\documentclass[twoside,11pt]{article}

\usepackage{natbib}

\usepackage{jmlr2e}

\usepackage{graphicx} 
\usepackage{subfigure}

\usepackage{enumerate}

\usepackage{algorithm}
\usepackage{algorithmic}

\usepackage{amsfonts}       
\usepackage{nicefrac}       
\usepackage{microtype,color,xcolor}      
\usepackage{tikz}
\usetikzlibrary{arrows,automata,shapes.misc}
\usepackage{amsmath}
\usepackage{paralist}
\graphicspath{{figs/}}

\usepackage{amsmath}
\usepackage{graphicx}
\usepackage{paralist}
\usepackage{booktabs}       
\usepackage{bm}

\usepackage{multicol}
\usepackage{multirow}
\usepackage{algorithm}
\usepackage{algorithmic}

\usepackage{hyperref}
\renewcommand{\cite}[1]{\citep{#1}}

\newcommand{\BV}{\bm{V}}

\newcommand{\omitme}[1]{}

\newcommand{\Mc}{\mathcal{M}}

\newcommand{\PP}{\mathbb{P}}

\newcommand{\NN}{\mathbb{N}}

\newcommand{\Rcal}{\mathcal{R}}

\newcommand{\Dcal}{\mathcal{D}}
\newcommand{\Vcal}{\mathcal{V}}

\newcommand{\defeq}{\mathrel{\mathop:}=}

\newcommand{\SPD}{\text{SPD}}

\newcommand{\EXP}{\text{Exp}}
\newcommand{\LOG}{\text{Log}}
\newcommand{\EE}{\mathbb{E}}
\newcommand{\x}{\textbf{x}}
\newcommand{\M}{\mathcal{M}}

\newcommand{\tr}{\text{tr}}

\DeclareMathOperator*{\argmin}{\arg\!\min}

\newcommand{\mV}{\mathcal{V}}
\newcommand{\Bbeta}{\bm{\beta}}

\newtheorem{assumption}[theorem]{Assumption}

\firstpageno{1}

\begin{document}

\title{Finding Differentially Covarying Needles in a Temporally Evolving Haystack: A Scan Statistics Perspective}

\author{\name Ronak Mehta\textsuperscript{1}  \email ronakrm@cs.wisc.edu \\
		\name Hyunwoo J. Kim\textsuperscript{1} \email hwkim@cs.wisc.edu \\
        \name Shulei Wang\textsuperscript{2} \email shulei@stat.wisc.edu \\
        \name Sterling C. Johnson\textsuperscript{3} \email scj@medicine.wisc.edu \\
	    \name Ming Yuan\textsuperscript{2} \email myuan@stat.wisc.edu \\
        \name Vikas Singh\textsuperscript{1,4} \email vsingh@biostat.wisc.edu \\
	 \AND
	    \addr \textsuperscript{1}Computer Sciences Department\\
	    \textsuperscript{2}Department of Statistics\\
		\textsuperscript{3}Department of Medicine\\
		\textsuperscript{4}Department of Biostatistics\\
        University of Wisconsin-Madison\\
        Madison, WI 53706, USA
     \AND
       \name Project webpage: \url{pages.cs.wisc.edu/~ronakrm/research/covtraj/} \\
      \name Video summary: \url{vimeo.com/205606140}
       }


\maketitle

\begin{abstract}
Recent results in coupled or temporal graphical models offer schemes for estimating the relationship structure 
between features when the data come from
related (but distinct) longitudinal sources. A novel application of these ideas is for analyzing group-level differences, i.e., in identifying if {\em trends} of estimated objects (e.g., 
covariance or precision matrices) are different across disparate conditions (e.g., gender or disease). Often, poor effect sizes make detecting the \textit{differential} signal 
over the {\em full} set of features difficult: for example, 
dependencies between only a {\em subset of features} may manifest differently across groups.
In this work, we first give a parametric model 
for estimating trends in the space of $\SPD$ matrices as a function of one or more covariates. We then generalize scan statistics to graph structures, 
to search over distinct subsets of features (graph partitions) whose temporal dependency structure may show statistically 
significant group-wise differences.
We theoretically analyze the Family Wise Error Rate (FWER) and bounds on Type 1 and Type 2 error. 
On a cohort of individuals with risk factors for Alzheimer's disease (but otherwise cognitively healthy), 
we find scientifically interesting 
group differences where the default analysis, i.e., models estimated on the {\em full graph}, do not survive reasonable 
significance thresholds. 

\end{abstract}


\clearpage

\section{Introduction}
Multivariate data analysis exploiting the conditional independence structure between features or covariates using 
undirected graphical models is now standard within any data analysis toolbox. 
When the data are multivariate Gaussian, the zeros in the inverse covariance (precision) matrix give conditional independences 
among the variables \cite{lauritzen1996graphical}. Further, if the precision matrix is sparse, we can  
derive dependencies between features when the data are high-dimensional and/or the number of measurements are small. 
The estimation of a graphical model
has been extensively studied
and a rich literature is available describing 
its statistical and algorithmic properties \cite{koller2009probabilistic,jordan1998learning}. 
For instance, the so-called \textit{graphical lasso} formulation uses an $\ell_1$-norm penalty on the 
precision matrix and is widely used, and consistency properties 
in the large $p$ regime \cite{cai2011constrained,friedman2008sparse,yuan2010high} are now well understood. 
These formulations have also been extended to various transformations of Gaussian distributions (e.g., non-paranormal)
using rank statistics \cite{liu2009nonparanormal,xue2012regularized,liu2012high}.

{\em Coupled and Temporal Graphical Models.} 
Often, data come from two (or more) disparate sources or multiple timepoints.
Within the last few years, a few proposals have 
described strategies for linking the sparsity patterns of multiple graphical models, e.g., using a fused lasso 
penalty \cite{danaher2014joint} \cite{yang2015fused}. Observe that 
if the data sources correspond to {\em longitudinal} acquisitions, we should expect 
the `structure' to gradually evolve.
Several authors have offered generalizations to address this problem: \cite{zhou2010time} removes the assumption
that each graph is independent and structurally `close'.
Instead, \cite{zhou2010time} can be thought of as a growth model \cite{mcardle2000introduction} defined on these structures: they show how non-identically distributed graphs can be learned over time. 
Recently, the nonparametric procedure in \cite{qiu2015joint} extends these ideas
to handle multiple sources, each with multiple samples.

The ideas in the literature so far to ``couple'' multiple graphical model estimation modules are mostly nonparametric. 
While such a formulation offers benefits, in many estimation problems, 
parametric models may 
be more convenient for downstream statistical analysis,
particularly for hypothesis testing \cite{hardle1993comparing,geer2000empirical,roehrig1988conditions}.
Given that the topic of \textit{coupled} graphical models, by itself, is fairly recent, algorithms for {\em parametric estimation} of 
temporal or coupled Gaussian graphical models have not yet been heavily studied. 
This will involve parameterizing {\em trends} in the highly structured nature of the `response' variable ($\SPD$ matrices). 
We find that parametric formulations for manifold-valued data {\em have} been proposed recently \cite{hjkimcvpr2014,cornea2016regression}. 
Because $\SPD$ matrices form a Riemannian manifold, algorithms
that estimate a parametric model respecting the underlying Riemannian metric are more suitable in many applications as opposed to assuming a Euclidean metric 
on positively or negatively curved spaces \cite{xie2010statistical, fletcher2007riemannian, jayasumanakernel}. We will make a few simple modifications 
(for efficiency purposes) to such algorithms and make use of the estimated parameters for follow-up analysis. 


{\em Finding Group-wise Differences.} Assuming that we have a black-box procedure to estimate a parametric model on the $\SPD$ manifold available, 
in many tasks, such an estimation is merely a segue to other analyses designed to answer scientifically meaningful questions. 
For example, we are often interested in asking whether the temporally coupled model estimated using the procedure above differs 
in meaningful ways {\em across} groups induced by a stratification or dichotomous variable (e.g., gender or disease). For instance, is the `slope' in structured response space statistically different 
across education level or body mass index? 
While the body of work for graphical model estimation is mature, the literature describing hypothesis tests in this
regime \cite{diffnet,belilovsky2015hypothesis}
is sparse at best.
Given that such questions are simpler to answer with alternative schemes (with assumptions on the distributional properties of the data), e.g., structural equation modeling, 
latent growth models and so on \cite{ullman2003structural, mcardle2000introduction}, it seems that 
the unavailability of such tools is limiting the adoption of such ideas in a broader 
cross-section of science. We will seek to address this gap. 

{\em Needles in Temporal Haystacks.} If we temporarily set aside the potential value of a hypothesis test framework for temporal 
trajectories in graphical models, we see that
from an operational viewpoint, such procedures are most effective when a practitioner already has a precise scientific question in mind. In reality, however, 
many data analysis tools are deployed for exploratory analyses to inform an investigator as to which questions to ask. 
Being able to ``localize'' which parts of the model are different across groups can be very valuable. This ability actually 
benefits statistical power as well. Notice that when the stratified groups are not very different 
to begin with, e.g., healthy individuals with presence or absence of a genetic mutation, the
effect sizes (statistical difference between two groups) are likely to be poor.
Here, while the trends identified on the {\em full} precision matrix may still be different (i.e., there may be a {\em real} signal 
associated with a grouping variable), 
they may not be strong enough to survive significance thresholds. Ideally, what we need here are analogs of the widely used ``scan statistics'' 
for our hypothesis testing formulations for temporal graphical models --- to identify which {\em parts of the signal} are promising. 
Then, even if only a small subset of 
features were different across groups,
we may be able to identify these differential effects efficiently. This benefits Type 2 error, 
provides a practical turnkey product for an experimental scientist, and makes up the key technical results of our work.


Briefly, we provide \textbf{(i)} a simple and efficient parametric procedure for modeling temporally evolving graphical models, \textbf{(ii)} a 
hypothesis test for identifying differences between group-wise estimated models, and \textbf{(iii)} a scan
algorithm to identify {\em those subsets of the features which contribute to the group-wise differences}.
Together, these ideas offer a framework for identifying group-wise differences in temporally coupled graphical models.
From the experimental perspective, we find scientifically plausible results on 
a unique longitudinally tracked cohort of middle-aged (and young elderly) persons at risk for Alzheimer's disease due to family history, 
but who are otherwise completely cognitively healthy.

The rest of the paper is organized as follows. In Section \ref{sec:mglm} we present an efficient manifold regression procedure for 
modeling covariance trajectories, which serves as a blackbox module in our hypothesis testing framework. 
In Section \ref{sec:hyp-test}, we define our main hypothesis test for group difference analysis over covariance trajectories. 
In Section \ref{sec:loc}, we present a set of technical results describing our localization procedure based on scan statistics, 
as well as derive suitable size corrections to compare across feature subsets. Sections \ref{sec:loceval}, \ref{sec:pipeval}, and \ref{sec:wrap} conclude with
empirical evaluations of our model on synthetic data, various types of demographics/behavior data collected longitudinally 
in the United States from publicly available resources, and finally, our 
analysis on a unique longitudinal dataset (followed since 2001) from a preclinical Alzheimer's disease study involving approximately 1500 individuals.

\section{Characterizing Covariance Trajectories}
\label{sec:mglm}
Our main statistical testing framework, to be described shortly, needs an efficient means for calculating a ``trajectory" of the feature-by-feature interaction graphs over time
for the given longitudinal data. We now describe a scheme which offers this capability. 
Let $X_t \in \mathbb{R}^{n_t,p}$ be the design matrix of all $n_t$ samples at time $t$, where $t \in \{1,\ldots,T\}$, and $T$ is the total number of distinct timepoints.
We wish to capture the trends in the relationships between the features as a function of $t$. 
To evaluate the groupwise differences in changes of such interactions, we make use of the fact
that these interactions are commonly captured by correlation or conditional independence, represented by the covariance matrix (with normalized features)
and the precision matrix (the inverse of covariance matrix).

Here we simply use the covariance matrix for each timepoint $t$ to denote the interaction between features, 
$C_t = cov(X_t)$. 
Our goal now is to estimate the parameters of the function, $t \to C_t$. 
%
%
We may vectorize the covariance matrix and apply a linear model; its parameters
will give the trajectory in ``vectorized covariance space'' as we scan through $t$. 
But these predictions are {\em not} guaranteed to be valid  $\SPD$ matrices and even if a projection is performed to obtain a covariance estimate, distortions introduced by the process may be significant \cite{fletcher2013geodesic}.
It is well known that classical vector space models tend to be suboptimal 
in the manifold setting (covariance matrices live on the $\SPD$ manifold)
since they use Euclidean metrics which are defined in the ambient space. For manifold-valued data, Riemannian metrics are shown to be superior in many applications 
\cite{fletcher2007riemannian,banerjee2015nonlinear,jayasumanakernel,tuzel2007human}, 
%
and are increasingly being deployed in machine learning/statistics. 
We will utilize an appropriate statistical model informed by the manifold-structure of the data and 
then derive a hypothesis 
testing procedure to detect groupwise difference in the changes of interactions between features in longitudinal analysis.
%
To do so, we first summarize basic differential geometry notations \cite{do1992riemannian,lee2012introduction} and 
then describe our models. If desired,
any other (efficient) manifold-valued linear model \cite{fletcher2013geodesic} can be substituted in; no 
change in the workflow is needed. A reader familiar with manifold regression algorithms may consider this module as a black-box and skip ahead to
Section \ref{sec:hyp-test} which uses the parameter estimates from this procedure. 

\subsection{Riemannian Geometry}

Let $\Mc$ be a \textit{differentiable (smooth) manifold} in arbitrary dimensions.
A differentiable manifold $\Mc$ is a topological
space that is locally similar to Euclidean space and has a globally
defined differential structure. 
%
%
A \textit{Riemannian manifold} is a differentiable manifold $\Mc$ equipped with a smoothly varying inner product.
The \textit{geodesic curve} is the locally shortest path, analogous to straight lines in $\mathcal{R}^{p}$ --- this geodesic 
curve will be the object that defines the trajectory of our covariance matrices in $\SPD$ space. 
Unlike the Euclidean space, note that there may exist multiple geodesic curves between two points on a curved manifold. 
So, the \textit{geodesic distance}
between two points on $\Mc$ is defined as the length of the {\em shortest} geodesic curve connecting two points.
The geodesic distance helps in measuring the error of our trajectory estimation (analogous to a Frobenius or $\ell_2$ norm based loss in the Euclidean setting).
The geodesic curve from $y_i$ to $y_j$  is parameterized by a tangent vector in the tangent space anchored at $y_i$ with an exponential map $\EXP(y_i,\cdot ): T_{y_i}\Mc \rightarrow \Mc$. 
The inverse of the exponential map is the logarithm map, $\LOG(y_i,\cdot):\M \rightarrow T_{y_i}\M$. These two operations move us back and forth between 
the manifold and the tangent space. For completeness, Table \ref{tab:comp1} shows corresponding operations in the Euclidean space and Riemannian manifolds.
\begin{table}[]
	{
		\begin{center}
			\begin{tabular}{| l | l | l | }
				\hline
				Operation & Euclidean & Riemannian  \\  
				\hline 
				Subtraction & $\overrightarrow{x_i x_j} = x_j - x_i$ & $\overrightarrow{x_i x_j} = \LOG(x_i,x_j)$ \\ 
				 Addition & $x_i + \overrightarrow{x_j x_k}$ & $\EXP(x_i,\overrightarrow{x_j x_k})$ \\     
				 Distance  & $\| \overrightarrow{x_i x_j} \|$ & $\|\LOG(x_i,x_j) \|_{x_i}$ \\ 
				Mean  & $\sum_{i=1}^{n} \overrightarrow{\bar{x}x_{i}}=0$ &  $\sum_{i=1}^{n} \LOG(\bar{x}, x_i)=0$  \\ 
				Covariance & $\EE \left [ (x_i - \bar{x})(x_i - \bar{x})^{T} \right ]$& $\EE \left [ \LOG(\bar{x}, x)\LOG(\bar{x}, x)^{T} \right ]$\\ [1ex] \hline 
			\end{tabular}
		\end{center}
	}
	\caption{\label{tab:comp1}\footnotesize Basic operations in Euclidean space and Riemannian manifolds.}
\end{table}
Separate from the above notation, matrix exponential (and logarithm) are simply $\exp(\cdot)$ (and $\log(\cdot)$).  
%
Finally, \textit{parallel transport} is a generalized parallel translation on manifolds. Given a differentiable curve $\gamma : \mathcal{I} \rightarrow  \Mc$, where $\mathcal{I}$ is an open interval, 
the parallel transport of $v_0 \in T_{\gamma(t_0)}\Mc$ along curve $\gamma$ can be interpreted as the parallel translation of $v_0$ on the manifold preserving its length and the angle between $v (t)$ and $\gamma$. 
The parallel transport of $v$ from $y$ to $y'$ is $\Gamma_{y\rightarrow y'}v$.

\subsection{Riemannian Manifold Regression}
Several regression models for manifold-valued data have been proposed recently, a majority of 
which are nonparametric \cite{jayasumanakernel,banerjee2015nonlinear}. 
Because of the longitudinal nature of our dataset (and recruitment considerations in neuroimaging studies),
sample sizes do not exceed a few hundred participants (typically much smaller). 
We have found that generally, in this regime, parametric methods are better suited and also offer other benefits for downstream applications. 
Next, we will give a simple parametric model for this problem. 
\begin{figure*}[t]
  \centering
  \includegraphics[width=0.49\textwidth,trim={10 40 10 220},clip]{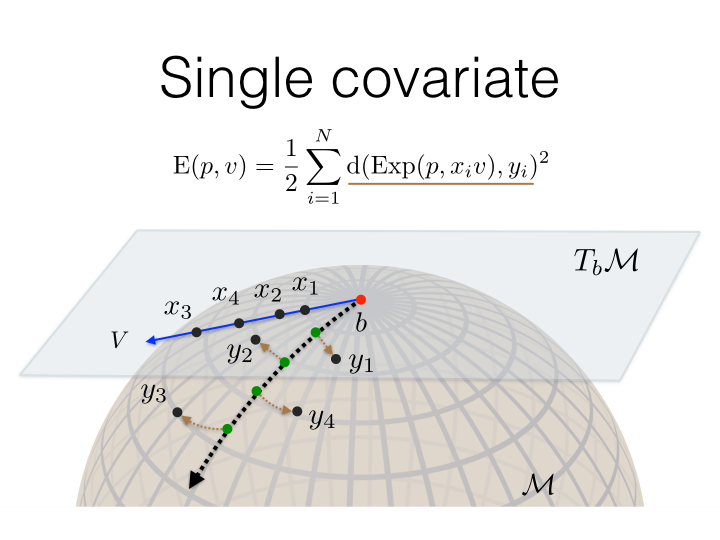}
    \includegraphics[width=0.49\textwidth,trim={10 40 10 220},clip]{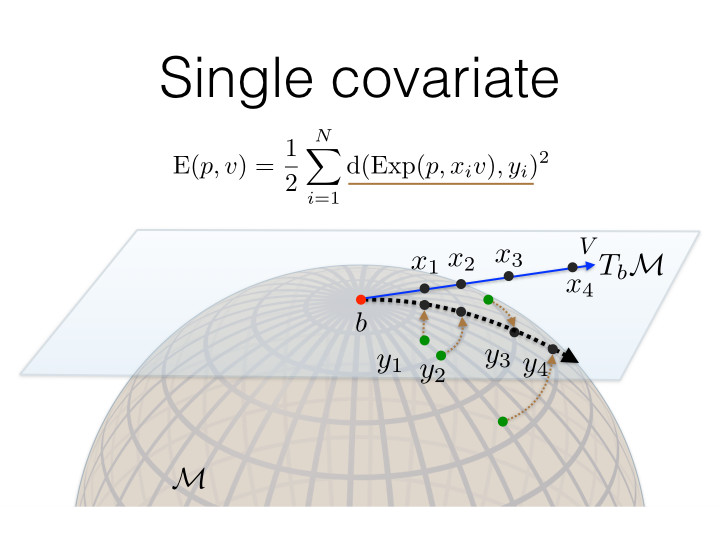}
  \caption{\label{fig:manifold}Group-wise MMGLM: The left and right figures represent two linear models on the $\SPD(p)$ manifold. Points $x_i$ in the tangent space are our covariate or predictor, and points $y_i$ in the manifold space represent $\SPD(p)$ matrices. In our regression setting, we wish to minimize the error (brown curves) between the estimation and the sample points. Because each linear model has a different base point, the trajectories cannot be directly compared as in the Euclidean setting.}
\end{figure*}
Let $x$ and $y$ be vectors in $\mathbb{R}^p$ and $\mathbb{R}^{p'}$ respectively.
\begin{definition} (Standard GLM.) The Euclidean multivariate multilinear model is 
{\begin{equation}
	\begin{split}
	y  = \beta^0 + \beta^{1} x^{1} + \beta^{2} x^{2} + \ldots +\beta^{p} x^{p} + \epsilon
	\end{split}
	\label{eq:generallinear}
	\end{equation}}
where $\beta^0$, $\beta^{i}$ and the error $\epsilon$ are in $\mathbb{R}^{p'}$ and $x = [x^1 \ldots x^p ]^{T}$ are the 
predictor variables.
\end{definition}
Henceforth, we will use the terms \textit{covariate} and \textit{predictor} interchangeably to describe those specific features we wish to control for in our model (e.g., time-points in our experiments).
For manifold-valued data, we adapt the formulation proposed by \cite{hjkimcvpr2014}.
\begin{definition} The Manifold Multivariate General Linear Model (MMGLM) is defined as 
{\begin{equation}
	\begin{split}
	&\min_{b \in \Mc, \forall j, V^j \in T_{b}\Mc} \quad \frac{1}{2} 
	\sum_{i=1}^{N}d(\EXP(b, \BV x_i),y_i)^2,
	\end{split}
	\label{eq:multigr}
	\end{equation}}
where $\BV x_i \defeq \sum_{j=1}^{n}V^j x_{i}^{j}$ and $d(\cdot, \cdot)$ is the geodesic distance between $\hat{y}_i:=\EXP(b, \BV x_i)$ and $y_i$. 
\end{definition}
This formulation generalizes \eqref{eq:generallinear}, by replacing the intercept $\beta^0$ and each vector $\beta^j$ for a covariate with a 
base point $b \in \Mc$ and a geodesic basis $V^j \in T_{b}\Mc$ respectively. The geodesic basis $V^j$ at $b$ parameterizes a geodesic curve $\EXP(b,V^jx^j)$.
Intuitively, this model is a `generalized' linear model with the inverse exponential map $\EXP^{-1}$ (or logarithm map $\LOG$) as a 
`link' function \cite{hjkimcvpr2014,cornea2016regression}. When the covariate/predictors are univariate, we will obtain a single geodesic curve, modeled via 
the so-called Geodesic Regression \cite{fletcher2013geodesic}.

\subsection{Efficient Estimation of Trajectories}
\label{sec:effest}
The objective in \eqref{eq:multigr}, can be solved by both gradient descent \cite{fletcher2013geodesic,hjkimcvpr2014} and MCMC methods \cite{cornea2016regression}. 
Unfortunately, these schemes can be expensive, especially when the dimension of the manifold is large. Further, if the algorithm needs to be run a 
large number of times, the computational footprint quickly becomes prohibitive. 
Motivated by these considerations, we use a so-called log-Euclidean approximate algorithm introduced in \cite{hjkimcvpr2014} with some adaptations, which requires mild assumptions on the manifold-valued data. 

Recall that in classical ordinary least squares (OLS), 
the regression curve goes through the mean of covariates and response variables, i.e., $y-\bar{y} = \beta(x-\bar{x})$.
Similarly, we assume that geodesic curves go through the mean of response variables on the manifold. Then, the base point, or intercept, ``$b$'' in \eqref{eq:multigr} can be approximated by the {\em manifold-valued mean of the sample points}, the Karcher mean \cite{karcher1977riemannian}. The propositions derived from \cite{hjkimcvpr2014} lead directly to the following. 
\begin{proposition}
Let $\bar{C}$ be the unique Karcher mean of a sufficiently close set of covariance matrices that lie on a curve $\Omega$. Then $\bar{C} \in \Omega$, and for some tangent vector $V \in T_{\bar{C}}\Mc$ and each $C$, there exists $x \in \mathbb{R}$ such that $C = \EXP(\bar{C},Vx)$. 	
\end{proposition}
This allows us to bypass the fairly involved variational procedure to estimate the base point $b$.

With this approximation of $\hat{b}$ via $\bar{y}$, the remaining variables to optimize are the tangent vectors $\BV$. 
We do so by taking advantage of log-Euclidean schemes. Once the base point is established as the Karcher mean, each data point on the manifold is projected into the tangent space at that point: $\LOG(\bar{y},y)$. These ``centered" points $\tilde{y}$ are now Euclidean, and if the covariates are centered as well ($\tilde{x}$), a closed form solution exists in the standard form of $\BV = \tilde{y}\tilde{x}^\top (\tilde{x}\tilde{x}^\top)^{-1}$. 

In this setting, it is often assumed that two points $y_1,y_2$ have a distance defined as $d(y_1,y_2) := \| \LOG(y_1,y_2) \|_{y_1} \approx \| \LOG(b,y_1)-\LOG(b,y_2) \|_{b}$. However,
on $\SPD$ manifolds with an affine invariant metric, each tangent space has a different inner product varying as a function of the base point $b$, i.e., $\langle  u,v\rangle_b:= \tr (b^{-1/2}ub^{-1}vb^{-1/2})$. This makes comparison of trajectories difficult without moving to tangent bundle formulations. This issue is discussed in some detail in \cite{muralidharan2012sasaki,hong2015group}. However, note that
\begin{remark}
When the base point $b$ is the identity $I$, then the inner product is exactly the Euclidean metric $\langle  u,v\rangle_b:= \tr (b^{-1/2}ub^{-1}vb^{-1/2})=\tr (uv)=\tr (u^Tv)$.
\end{remark}
This follows from the fact that $u$ and $v$ are symmetric matrices on $\SPD(p)$. We take advantage of this property
through \textit{parallel transport}. Specifically, we can bring all of the data to $T_{I}\Mc$ which will allow for a meaningful comparison of two tangent vectors from different base points.
Similar schemes have been used for projection on submanifolds in \cite{xie2010statistical} and other problems \cite{sommer2014optimization}. 
%
With a fast algorithm to compute \eqref{eq:multigr} available, we can now accurately model longitudinal trajectories of covariances matrices.
Our statistical procedure described next simply assumes 
the availability of some suitable scheme to solve the manifold-regression as defined in \eqref{eq:multigr} efficiently and does not
depend on particular properties of the foregoing algorithm. 

\section{Test Statistics for $\SPD(p)$ Trajectories}
\label{sec:hyp-test}
With an algorithm to construct a regression model for covariance matrix responses in hand, we can now describe a key
component of our contribution: a test statistic which allows addressing the main question of interest: 
{\em Is the progression/trajectory of covariance matrices (over time) different across two groups?} In the standard two-sample testing problem, a hypothesis test is set up
to check if the parameters of each group are significantly different:
\begin{align}\label{eq:hyptest}
H_0: \theta_1 = \theta_2 \quad vs. \quad H_A: \theta_1 \neq \theta_2
\end{align}
Recall that in a general linear model (GLM), when testing for mean group differences, the test parameters are the regression slopes from a standard GLM fit. 
In our setting, the parameters of interest are the population covariance trajectories estimated from the manifold regression in \eqref{eq:multigr}, see Fig. \ref{fig:manifold}. 
While the trajectories and the slopes are related, note that our parameters are estimated {\em on the manifold}. 
Two unique manifold trajectories, when projected as simple multivariate responses in Euclidean space, may not be significantly different under the GLM hypothesis testing framework, as has been observed by \cite{du2014geodesic}. Returning to our longitudinal trajectory formulation, we have the following na\"ive Covariance GLM:
\begin{definition} Let $vec(C_{g,t})$ be the vectorized covariance matrix at timepoint $t$ for group $g \in \{1,2\}$. Then the 
na\"ive Covariance GLM is defined as
\begin{align}\label{eq:euclideanGLM}
vec(C_{g,t}) = \beta_g^{0} + \beta_g t + \epsilon
\end{align}
with the slope $\theta = \beta$ in the hypothesis test in \eqref{eq:hyptest}, and $vec(\cdot)$ is the vectorized form of the input matrix. 
\end{definition}
With this model, our hypothesis testing reduces to a simple difference of slopes, which is well-studied in classical statistics literature.
\begin{definition}\cite{seber2003linear}
Let $\Bbeta_1,\Bbeta_1$ be the multivariate slopes calculated from estimating \eqref{eq:euclideanGLM}. Then an $\alpha$-level hypothesis test rejects the null hypothesis $\Bbeta_1 = \Bbeta_1$ when $L > \chi^2_{p}|_{1-\alpha}$, where
\begin{equation}
L = (\hat{\Bbeta}_1 - \hat{\Bbeta}_2)\Sigma^{-1}(\hat{\Bbeta}_1 - \hat{\Bbeta}_2)
\end{equation}
\label{eq:euclideanhyptest}
\end{definition}
Knowing that the response space is structured, i.e., our covariance matrices lie on the $\SPD$ manifold, we seek a more appropriate test and corresponding test statistic which 
adequately captures this knowledge. 

Observe that we can directly apply the manifold regression in \S \ref{sec:mglm} to solve for a linear model on the manifold. 
That is, we construct the manifold GLM as
\begin{definition} Let $C_{g,t}$ be the covariance matrix at timepoint $t$ for group $g \in \{1,2\}$. Then the 
	Longitudinal-Covariance GLM (LCGLM) is defined as
	\begin{align}\label{eq:LCGLM}
		C_{g,t} = \EXP(b_g, \BV_g t)
	\end{align}
	with $b_g$ and $\BV_g$ being the base point and tangent vector respectively, as described in \S\ref{sec:mglm}.
\label{eq:lcglm}
\end{definition}
But instead of solving $p(p-1)/2$ independent regressions, now we must concurrently solve for the entire manifold-valued response variable.
In this case, we cannot directly compare our trajectories because they lie in {\em different} tangent spaces. To accurately compare two tangent vectors, 
we must parallel transport both vectors to the same tangent space. Once they are both in the same space, we can construct a simple test statistic for the trajectory difference.
\begin{equation}
L = \|\Gamma_{b_1 \rightarrow I} \BV_1 - \Gamma_{b_2 \rightarrow I} \BV_2 \|_{I}^2
\label{eq:traj}
\end{equation}
Recall that the inner product at the Identity $I$ coincides with the Euclidean metric. This can now be naturally interpreted as a difference of slopes, and together with a standard Euclidean Normal noise assumption yields the following hypothesis test.
\begin{proposition}\label{prop_prodstat}
Assume that $\Gamma_{b \rightarrow I} \BV$ is normally distributed $N(0,I)$. Then the statistic defined in \eqref{eq:traj} follows a $\chi^2_{p}$ distribution with $p$ degrees of freedom, and the threshold test in \eqref{eq:euclideanhyptest} is an $\alpha$-level hypothesis for the covariate trajectory group difference.
\end{proposition}

\subsection{Incorporating First-Order Differences}
In many real world situations, first-order information in the data is often valuable in identifying group differences. Restricting 
our analysis to only the second-order interactions, i.e., covariances, may be inefficient (or sub-optimal) when the mean signal difference between groups is large. Our construction easily extends to these cases. Particularly, the \textit{product space} over both means and covariances is in $\mathbb{R}^{p} \times \SPD{(p)}$. 
\begin{remark}
The typical GLM on the first order information is defined in the standard Euclidean space. So, computing the regression in the product space 
$\mathbb{R}^{p} \times \SPD{(p)}$ 
amounts to simply 
computing the regression on the first and second order statistics (mean and covariance) separately.
\end{remark}
The above statement suggests that by applying the manifold regression to the covariances and the standard regression model for the means, 
we are directly solving the product space regression problem, incorporating both first and second order statistics.
However, in these cases, the statistic defined above in \eqref{eq:traj} does {\em not} directly take into account the potential difference in means. 
However, given our Normal noise assumption we can easily invoke the standard Gaussian multivariate likelihood statistic for group differences.
\begin{definition}
Let $\hat{\mu}_t, \hat{\Sigma}_t$ be the estimated mean and covariance from the standard linear model and our manifold-covariance GLM respectively. Then the Gaussian likelihood of our data $X$ is
\begin{align}
P(X|\hat{\mu},\hat{\Sigma}) = \prod_{t=1}^{T}\prod_{i=1}^{n_t} P(X_t | N(\hat{\mu}_t,\hat{\Sigma}_t)),
\end{align}
where $X_t$ is the subset of our data collected at timepoint $t$. Additionally, we can define a standard likelihood ratio test statistic as:
\begin{align}
L_{prod} = \frac{P(X_1|\hat{\mu}_1,\hat{\Sigma}_1)P(X_2|\hat{\mu}_2,\hat{\Sigma}_2)}{P(X_{1,2}|\hat{\mu}_{1,2},\hat{\Sigma}_{1,2})}
\end{align}
\end{definition}
This statistic is again $\chi^2_{p}$-distributed \cite{seber2003linear}, and an $\alpha$-level hypothesis test for group difference analysis 
can be defined in the same way as above. While our manifold regression modeling is focused on the case of centered data (where the mean signal may not be 
significantly different between the groups), we use the product space construction, wherever appropriate, in experimental evaluations.

\section{Localizing Group Differences for $\SPD(p)$ Trajectories}
\label{sec:loc}

The above procedure provides a precise mechanism to derive a statistic from the group-wise covariance matrix trajectories.
However, when the effect sizes are poor, any scheme operating on 
the trajectories of the {\em full covariance matrix} may still fail to identify group differences (as is the case in our experiments). To improve statistical power, localizing the process of computing the trajectories {\em 
  only to the relevant features} is critical. 
%
%
To this end, we consider the following global hypothesis testing problem
\begin{equation*}
H_0: \forall R, \Bbeta_1^R=\Bbeta_2^R \quad vs. \quad H_1: \exists R, \Bbeta_1^R\ne\Bbeta_2^R,
\end{equation*}       
where $\beta$ denotes the slope and $R$ is the region of the covariance matrix
{\em which only includes the relevant features}, see Fig. \ref{fg:ball}.
It turns out that by adapting {\em Scan statistics} \cite{fan2012control, arias2011detection}, we will be able to exclude the effect of irrelevant regions of the covariance 
matrix in the calculated trajectories. 
By extending this concept to graphs, we obtain an algorithm to identify {\em subsets of features} of the covariance matrix which show group differences that 
are otherwise unidentifiable, in a statistically rigorous way. 

\subsection{Scan Statistics}
Scan statistics are a valuable tool for structured multiple testing.
In its simplest form, we can consider a setting where we place a window (or box) over a region $R$ in an image and calculate 
a local statistic $L_R$, e.g., an average or a response to a convolution filter. 
Then, the window can be raster scanned at various locations in the image ($\Rcal$) and the maximum 
over the set of local statistics can be called the scan statistic. 
Intuitively, if the image is assumed to be a Gaussian random field, we can set up a null hypothesis using a
critical value and finding a statistically significant signal (i.e., regions) corresponds to comparing 
the local region-wise statistic with the critical value.  
Of course, there is flexibility in terms of specifying properties of the regions as described next. 
 
\begin{definition} Let $\Rcal$ be the collection of all possible structured regions, and $L_R$ be some statistic over region $R$, a structured subset of $\Rcal$. The scan statistic is defined as $L^* = \max_{R\in\Rcal} L_R$.
\end{definition}
Recent results in scan statistics show how \textit{size corrections} can be used to increase detection power in multi-scale analysis with 
nice guarantees \cite{walther2010optimal,wang2016structured}. 
%
To utilize these ideas for our hypothesis test, we must extend scan statistics and these size corrections to a graph setting 
where the graph is induced by a sparse estimation of the precision matrix, e.g., graphical lasso (or any other algorithm of choice) over the features.
To do so, structured regions $R$ and a statistic $L_R$ on each region must be defined on the graph. Intuitively, in our case, 
$L_R$ must capture the ``difference'' in group-wise covariance trajectories. 
As we will describe shortly, it is in the context of this statistic where we utilize the LCGLM \eqref{eq:LCGLM}, which will be invoked at the {\em level of individual regions} $R$, one by one. 

Let $G:=(\mV,E)$ be a graph over the features (represented in the covariance matrix) with vertex set $\mV$ and edge set $E$. 
We define the structured region $R \subseteq G$ as a connected subgraph of $G$ corresponding to the 
selection of those vertices as our feature subset (block of the covariance matrix, see Fig. \ref{fg:ball}). 
A natural question is whether such an enumeration is tractable 
if the number of connected subgraphs $\Rcal$ is exponential. It turns out that if we make a mild assumption on the graph, 
the number of induced regions can be shown to be polynomially bounded. Further, it then naturally provides a {\em size correction}, the analog for a multiple testing adjustment. 

{\em Remarks.} In our motivating application, the group differences we seek to identify will involve a cohesive set of features that will be 
connected to each other, by definition (large changes in covariances indicate dependent features). 
Based on this observation, we assume that the true localized subgraph is a ``ball'' subgraph. 
\begin{definition} A ball subgraph consists of nodes with a given radius $r$ from a particular node (see Fig. \ref{fg:ball}). The collection of ball subgraphs is defined as
\begin{equation}
\Rcal=\left\{B(v,r):v\in \mV\ {\rm and}\ r\in \NN\right\}
\end{equation}
where the ball subgraph $B(v,r):=\{v^\prime \in V:d(v,v^\prime)\le r\}$, and $d(v,v^\prime)$ is the minimum length path connecting $v$ and $v^\prime$.
\end{definition}
%
With this assumption, it can be verified that we now only need to search a polynomially bounded set of regions.
\begin{remark}
The number of unique ball subgraphs in any graph $G$ is bounded above by $D|\mV|$, where $D$ is the diameter (longest chain) of the graph $G$.
\end{remark}
 On these regions 
(i.e., blocks of covariance matrix), 
we will invoke LCGLM to provide us a statistic $L_R$. This is just the difference in slopes of the calculated manifold regression across groups in \eqref{eq:traj}. 
We will iteratively obtain this statistic for distinct regions $R$ and find subgraphs that differ in their trajectories across groups using a size correction for hypothesis tests.
\begin{figure}[]
	\begin{center}
		{
			\begin{tikzpicture}
			\path[draw,line width=2pt,black!70!white] (-0.5,-0.5) -- (2,-0.5)--(2,2)--(-0.5,2)--cycle; 
			\fill[black!30!white,fill opacity=0.8] (0.75,0.75) -- (0.75,1.25) -- (0.25,1.25) -- (0.25,0.75) --cycle ; 
			\path[draw,line width=2pt,black!70!white] (-0.5,2)--(2,-0.5);
			
			\draw[-latex,thick,black!70!white,line width=3pt](2.5,1)
			to[out=0,in=180] (3.0,1);
			
			\node[shape=circle,draw=black,line width=2pt] (A) at (4,-0.2) {A};
			\node[shape=circle,draw=black,line width=2pt] (F) at (3.7,1) {F};
			\node[shape=circle,draw=black,line width=2pt] (C) at (6,1.2) {C};
			\node[shape=circle,draw=black,line width=2pt] (D) at (6,-0.2) {D};
			\node[shape=circle,draw=black,line width=2pt] (E) at (5,2) {E};
			\path [draw,line width=2pt,black!70!white] (A) edge node {}  (F);
			\path [draw,line width=2pt,black!70!white] (A) edge node {}  (C);
			\path [draw,line width=2pt,black!70!white] (E) edge node {}  (F);
			\path [draw,line width=2pt,black!70!white] (E) edge node {}  (C);
			\path [draw,line width=2pt,black!70!white] (C) edge node {}  (D);
			
			\draw[-latex,thick,black!70!white,line width=3pt](7.0,1)
			to[out=0,in=180] (7.5,1);
			
			\node[shape=circle,draw=black,line width=2pt] (B2) at (8.2,1) {F};
			\node[shape=circle,draw=black,line width=2pt] (C2) at (10.5,1.2) {C};
			\node[shape=circle,draw=black,line width=2pt] (E2) at (9.5,2) {E};
			\path [draw,line width=2pt,black!70!white] (E2) edge node {}  (B2);
			\path [draw,line width=2pt,black!70!white] (E2) edge node {}  (C2);
			
			\draw[thick,red] ([shift=(-40:2)]9.5,2.1) arc (-40:-152:2.1);
			\draw[thick,red] ([shift=(-40:2)]5,2) arc (-40:-150:2);
			\draw[thick,red] ([shift=(-30:0.5)]5,2) arc (-30:-170:0.5);
			\draw[thick,red] ([shift=(-50:2.8)]5,2) arc (-50:-140:2.8);
			
			\draw[thick,black](5.5,1.9)node[right]{$r=0$};
			\draw[thick,black](6.4,0.6)node[right]{$r=1$};
			\draw[thick,black](6.8,0)node[right]{$r=2$};
			\draw[thick,black](10,0.1)node[right]{$r=1$};
			\draw[thick,black](7.9,-0.5)node[right]{$B(E,1)=\{F,C,E\}$};
			
			\end{tikzpicture}
		}
	\end{center}
	\caption{\label{fg:ball} (left) A region of the sparse precision matrix, (center) The corresponding subgraph of that region, along with balls of varying radius from the root node $E$, (right) The ball subgraph constructed with $r = 1$.  These subgraphs with bounded radius act as the 
		structured regions on which scan statistics can be applied.}
\end{figure}
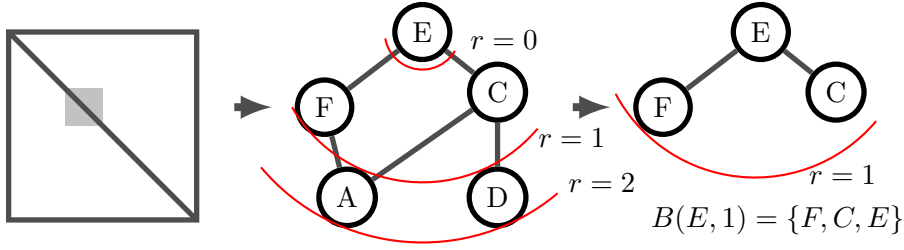

Let us revisit the standard linear model setting and assume that 
our slopes $\Bbeta_g^R$ correspond to the subset of slopes from features in $R$, and $\hat{\Bbeta}_g^R$ is an estimate of that slope. In this case, we have the following  statistic (see e.g. \cite{seber2003linear}),
\begin{align}
(\hat{\Bbeta}_1^R - \hat{\Bbeta}_2^R)\Sigma_R^{-1}(\hat{\Bbeta}_1^R - \hat{\Bbeta}_2^R) \sim \chi_{|E(R)|}^2,
\end{align}
where $\Sigma_R^{-1}$ is the covariance matrix of $\hat{\Bbeta_1^R} - \hat{\Bbeta_2^R}$. With a normal noise assumption, this covariance will 
be identity and the statistic would simply be the $\ell_2$-norm difference as in the classical analysis. 
To make the statistics comparable across {\em different sizes}, we use the standardized version of a $\chi_{|E(R)|}^2$ distribution,
{\small\begin{align}\label{eq:LRstat}
L_R = \frac{(\hat{\Bbeta}_1^R - \hat{\Bbeta}_2^R)\Sigma_R^{-1}(\hat{\Bbeta}_1^R - \hat{\Bbeta}_2^R) - E(R) }{ \sqrt{E(R)} }.
\end{align}}
We can extend this analysis to our manifold setting.
\begin{definition}
For a given structured region $R$, the region-based LCGLM is written as
{ \begin{equation}
(b_g^R,\BV_g^R) = \argmin_{(b^R,\BV^R) \in T\Mc^R} 
	\EE \left[ d(\EXP(b^R, \BV^R t_g),C_g^R)^2\right]
\end{equation}}
where $C_g^R$ is the covariance matrix subblock defined by features included in $R$ for group $g$ ($t_g$ is our univariate predictor, i.e., time).
\end{definition}
To compare the group trajectories, we first parallel transport each tangent vector to the identity as described in \S\ref{sec:mglm} and then compute the statistic in \eqref{eq:traj} given as $\|\Gamma_{b_1^R \rightarrow I} \BV_1^R - \Gamma_{b_2^R \rightarrow I} \BV_2^R \|_{I}^2$.
In the case of the product space construction, we apply the test in \eqref{prop_prodstat} to the data subset corresponding to the features 
in region $R$, with the same correction as in \eqref{eq:LRstat}.

{\em Summary.} We now have a region-based statistic for the manifold regression setting that is approximately normally distributed $N(0,1)$, allowing effective comparison across differently-sized regions.

\subsection{Size Correction}
A final unresolved yet important issue is that we must correct $L_R$ based on the number of edges $E(R)$ in $R$.
This has a direct consequence on detection power.
Observe that the normalization for size correction should be determined by the null distribution of $L_R$, i.e., 
when there is no slope difference in the trajectories between groups.
In order to derive a correction, we need to characterize the behavior of scan statistics within roughly similar regions, $\max_{R\in \Rcal(A)}L_R$, 
where $\Rcal(A)$ is the collection of region $R$s with similar size as $E(R)$,
\begin{align}
\Rcal(A)=\{R\in \Rcal:A/2<|E(R)|\le A\}.
\end{align}
Clearly, the behavior of $\max_{R\in \Rcal(A)}L_R$ depends on the ``complexity'' of $\Rcal(A)$. 
A clear understanding of how similar subgraphs relate to each other leads directly to a correction tied to their relative sizes.

To investigate the complexity of $\Rcal(A)$, we define the following quantities.
\begin{definition} The distance between subgraphs $R_1$ and $R_2$ can be given as
{\begin{align}
d(R_1,R_2)=1-\frac{|E(R_1)\cap E(R_2)|}{ \sqrt{|E(R_1)||E(R_2)|}}
\end{align}}
\end{definition}
\begin{definition}
Let the $\epsilon$-covering number of $\Rcal(A)$, denoted by $N(A,\epsilon)$, be the smallest integer such that there
is a subset $\Rcal_{approx}(A,\epsilon)$ of $\Rcal$ such that
{\begin{align}
\sup_{R_{1}\in\Rcal(A)}\inf_{R_{2}\in \Rcal_{approx}(A,\epsilon)}d(R_{1},R_{2})\le \epsilon
\end{align}}
where $|\Rcal_{approx}(A,\epsilon)|=N(A,\epsilon)$.
\end{definition}
We can verify that all regions in $\Rcal(A)$
can be approximated by regions in $\Rcal_{approx}(A)$ with reasonably small error.
From the definitions, notice that the complexity of $\Rcal(A)$ is reflected by $N(A,\epsilon)$.
If $N(A,\epsilon)$ is nicely bounded (as is the case here), scan statistics can be calculated very efficiently 
(Lemma \ref{lm:entropy}). 

Before stating this result, we make a mild assumption on our graph. 
For any ball subgraph, the edges around its center are not too sparse, compared to the edges 
in the outer region of the ball subgraph, i.e., hard on the inside, soft on the outside. This yields, 
\begin{assumption}(Avocado) There exist constants $S$ and $H$ such that, for any $r/2 \leq r' \leq r$ and $v \in \mV$,
{\begin{align}\label{eq:lightskirt}
\frac{|E(B(v,r'))|}{|E(B(v,r))|} \geq H\left(1 - \frac{|E(B(v,r-r'))|}{|E(B(v,r))|}\right)^S.
\end{align}}
\end{assumption}
We see that this assumption holds for many classes of graphs:
a ring graph satisfies this condition when $H=1$ and $S=1$ and the 2-d lattice satisfies this condition when $H=1/4$ and $S=2$ (see Fig. \ref{fg:avocado}). 
\begin{figure}[]
	\begin{center}
		\scalebox{0.8}
		{
			\begin{tikzpicture}
			\tikzstyle{vertex}=[circle,draw,minimum size=0.2cm]
			\tikzstyle{every node}=[vertex]
			
			\node (A) { };
			\node (B) [below of = A] { };
			\node (C) [below of = B] { };
			\node (D) [below of = C] { };
			\draw (A) -- (B);
			\draw (B) -- (C);
			\draw (C) -- (D);

			\end{tikzpicture}
			\hspace{0.2cm}
			\begin{tikzpicture}
			\def \n {8}
			\def \radius {1.5cm}
			\def \margin {8} 
			
			\foreach \s in {1,...,\n}
			{
				\node[draw, circle] at ({360/\n * (\s - 1)}:\radius) { };
				\draw[-, >=latex] ({360/\n * (\s - 1)+\margin}:\radius) 
				arc ({360/\n * (\s - 1)+\margin}:{360/\n * (\s)-\margin}:\radius);
			}
			\end{tikzpicture}
			\hspace{0.2cm}
			\begin{tikzpicture}
			\tikzstyle{vertex}=[circle,draw,minimum size=0.2cm]
			\tikzstyle{every node}=[vertex]
			\foreach \x in {1,2,3,4}
			{
				\foreach \y in {1,2,3,4}
				{
					\node (\x \y) at (\x,\y) { };
					\ifnum\y>1
					\pgfmathparse{int(\y-1)}  
					\draw  (\x \y)  -- (\x \pgfmathresult);
					\fi
					\ifnum\x>1
					\pgfmathparse{int(\x-1)}                
					\draw (\x \y) -- (\pgfmathresult \y);
					\fi 
				}
			}   
			\end{tikzpicture}
			\hspace{1cm}
			\begin{tikzpicture}
			\def \n {8}
			\def \radius {1.5cm}
			
			\node[draw, circle] (ss) at ({0}:0) { };
			
			\foreach \s in {1,...,\n}
			{
				\node[draw, circle] (\s) at ({360/\n * (\s - 1)}:\radius) { };
				\path (ss) edge [] node { } (\s);
			}
			\path[draw,line width=0.5pt,red!70!white] (-1.5,1.5)--(1.5,-1.5);
			\path[draw,line width=0.5pt,red!70!white] (-1.5,-1.5)--(1.5,1.5);
			\end{tikzpicture}
		}
	\end{center}
	\caption{\label{fg:avocado} (Left) Chain, ring and 2D lattice graphs that satisfy the Avocado Assumption. (Right) Star graph that does not satisfy the property: from the center node the graph is ``too dense on the outside."}
\end{figure}
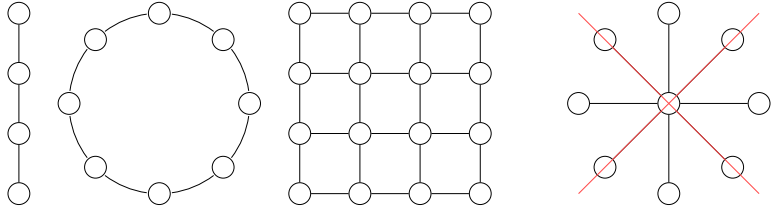
With this assumption, we have the following result for the $\epsilon$-covering number $N(A,\epsilon)$.
\begin{lemma}
	\label{lm:entropy}
        Let $|E|$ be the total number of edges in $G$. If (\ref{eq:lightskirt}) holds and $A$ is given, 
        then, for a constant $C_{H,S}$ which only depends on $H$ and $S$ in (\ref{eq:lightskirt}),
	{\begin{equation}
	    \label{eq:entropy}
	    N(A,\epsilon)\le C_{H,S}\frac{|E|}{ A}\left(\frac{1}{ \epsilon}\right)^{S+1}.
	\end{equation}}
\end{lemma}
The proof of this result follows from our ball-subgraph construction and our Avocado assumption and provided in the Appendix.

Intuitively, this result upper bounds the number of graphs that are necessary to search over to completely exhaust the search space of subgraphs. With this result, we can now construct a suitable size correction. 
Following the work of \cite{davies2001local} and \cite{wang2016structured}, we can increase the power of our test by using the following statistic:
{\begin{align} \label{eq:scanstat}
T^\ast=\max_{R\in\Rcal}\left(L_R-2\sqrt{\log\frac{|E|}{ |E(R)|}}\right).
\end{align}}
The significance of this size correction is that we now have \textit{a single critical value for each candidate subgraph, regardless of the subgraph size}. Our final test is defined as $\mathbb{I}[T^* > q_\alpha]$, where $q_\alpha$ is the $\alpha$-level quantile of $T^*$ under the null hypothesis (that no region is truly significant across groups). By construction, we can control the type 1 error at a specified $\alpha$-level.

Under the alternative hypothesis of this framework, it is important to note that in many cases,
large subgraphs that subsume smaller significant graphs may also have large test statistics, and our hypothesis test only indicates the existence of {\em some} significant region. To identify or localize the smaller subsets, we follow the procedure from \cite{jeng2010optimal}, by beginning with the subgraph with the largest test statistic and iteratively removing overlapping subsets from the total set of subgraphs. This requires testing each regional/local statistic, $(L_R - 2\sqrt{\log(|E|/|E(R)|)})$ against $q_\alpha$. Under this procedure, we can control the weak family-wise error rate (wFWER) if we view our problem via the lens of multiple testing. The weak FWER is the probability of false discovery under the null hypothesis. To see that this is inherently controlled, note
\begin{align}
\mathbb{P}(FN \geq 1|H_0) = \mathbb{P}(T^* > q_\alpha|H_0) \leq \alpha, 
\end{align}
where $FN$ is the number of false discoveries under the null hypothesis. With this correction at the group difference level, we completely avoid any multiple comparisons issues that would arise in the case of a test for each subgraph.
In addition to controlling the false positive rate, we have the following guarantee on {\em identifying truly significant regions} under the normal noise assumption.
\begin{theorem}
If \eqref{eq:lightskirt} holds and the number of edges in the candidate subgraph is larger than $\log^2 |E|$, i.e.,
\begin{equation}
\label{eq:setsize}
|E(R)|\gg \log^2 |E|\qquad \forall\ R\in\Rcal,
\end{equation}
then the critical value $q_\alpha$ satisfies
\begin{equation}
\label{eq:criticalvl}
q_\alpha=O(1).
\end{equation}
Moreover, as $|E|\to \infty$, if a subgraph $R_0$ obeys 
\begin{equation}
\label{eq:sigcond}
{(\Bbeta_1^{R_0}-\Bbeta_2^{R_0})^T\Sigma_{R_0}^{-1}(\Bbeta_1^{R_0}-\Bbeta_2^{R_0})\over \sqrt{|E({R_0})|}}\gg 2\sqrt{\log{|E|\over |E({R_0})|}},
\end{equation}
then as $|E|\to \infty$,
\begin{equation}
\label{eq:power}
\PP\left(L_{R_0}-2\sqrt{\log{|E|\over |E({R_0})|}}>q_\alpha\right)\to 1.
\end{equation}
\label{lm:mainthm}
\end{theorem}
The full proof of this result follows a generic chaining argument (see, e.g. \cite{talagrand2006generic}) along with application of 
concentration inequalities and union bounds, and can be found in the Appendix.

{\em Summary.}
At a high level, this result directly characterizes the behavior of $T^*$ under the null hypothesis $H_0$ and 
the alternative hypothesis $H_1$, respectively. 
We see that \eqref{eq:criticalvl} implies that $T^*$ can roughly be seen as a constant under the null hypothesis, 
and under the alternative hypothesis when \eqref{eq:sigcond} is satisfied, the test based on $T^*$ is consistent, see \eqref{eq:power}. 

\subsection{Workflow for conducting hypothesis tests on temporal trends of graphs}
With these guarantees, our full workflow is as follows. First, we use an oracle procedure to generate a graph over our features 
that roughly captures the conditional independences. Any procedure that provides a conditional independence graph is sufficient. 
Next, for each ball subgraph over this graph, we compute the Longitudinal-Covariance GLM over these features for both groups, and compute the statistics outlined in \S\ref{sec:hyp-test}. We then compute the size-corrected statistic, and compare against the single critical value. For all regions that pass this threshold, we apply the procedure from \cite{jeng2010optimal}. This workflow shows how to conduct hypothesis tests on temporal trends of large covariance matrices, with improved 
power and bounded Type 1 error. Additional implementation details can be found in the Appendix.

\section{Localization Evaluation: Trends of Tobacco Usage Across Gender}
\label{sec:loceval}
We begin our empirical analysis of the model by first applying the subgraph localization procedure by itself (standalone), separate from our manifold regression scheme. In this case, our statistic is derived from {\em only} Generalized Linear Models (GLM) constructions, where the $\hat{\Bbeta}_g^R$ in equation \eqref{eq:LRstat} is the slope estimated from fitting standard first order linear models. Identifying the differentially varying subgraphs across groups in this way is similar to a simpler version of the planted clique identification problem \cite{arora2009computational}, where the clique we are trying to identify corresponds to those nodes whose slopes vary significantly across groups.
\begin{figure}[]
	\begin{center}
		\includegraphics[trim={0cm 0cm 0cm 0cm}, clip, width=0.9\textwidth]{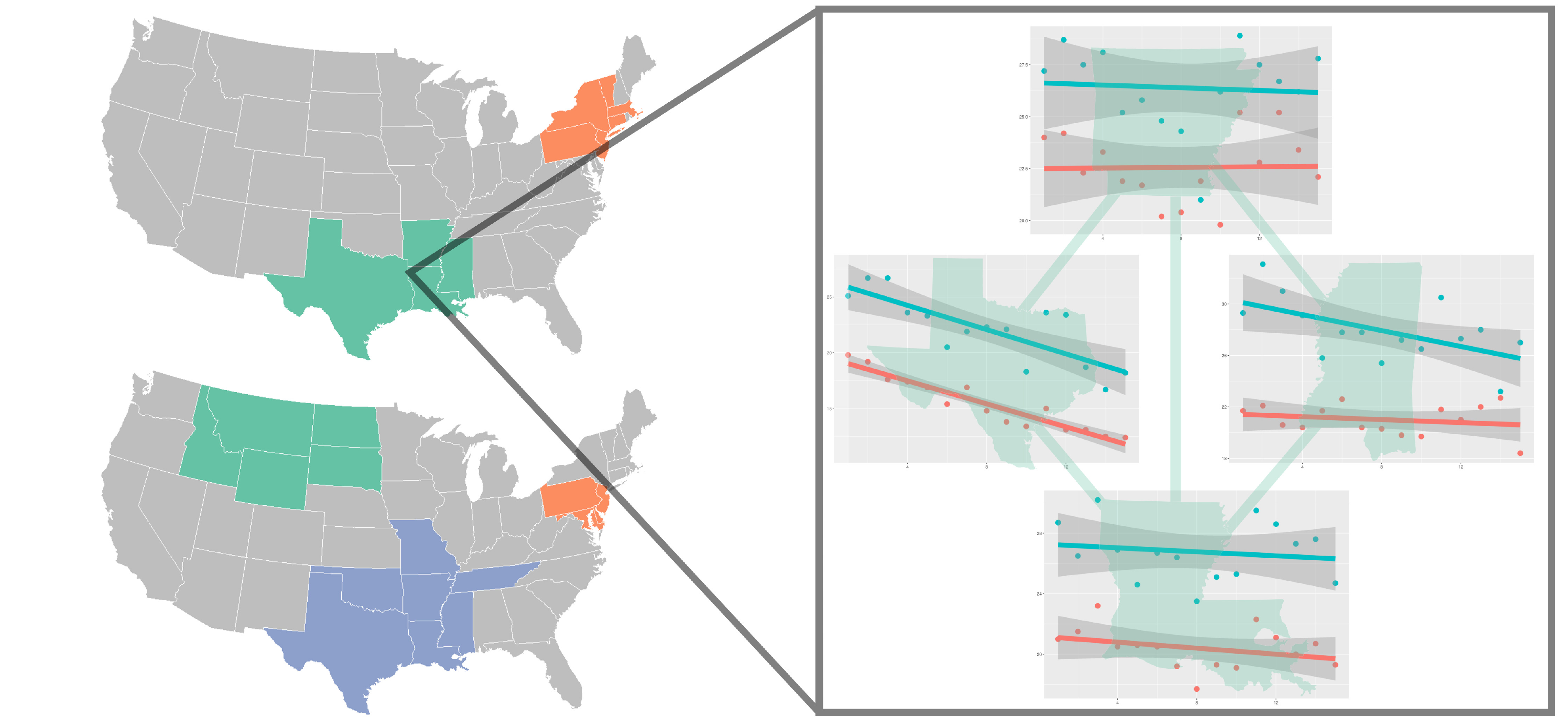}
	\end{center}
	\caption{(left,top) States identified as having significantly different time-varying tobacco usage across gender from 2001 to 2015. (left,bottom) States identified as having significantly different time-varying heavy drinking use across gender from 2010 to 2015. (right) Linear regressions over tobacco usage fitted to the four states defined by the ball subgraph centered at Louisiana. Best viewed in color.}
	\label{fig:tobalc}
\end{figure}

{\em Data.} The Center for Disease Control (CDC) provides extensive statistics regarding tobacco and alcohol usage across the US. This data has been collected systematically for the last few decades and is publically available (includes demographic information and gender). As a simple application of our 
proposed framework, 
we may pose the following question: which ``sub-groups'' of states tend to evolve differently in their correlation (pertaining to tobacco/alcohol usage) over time? 
Our framework extends easily to answer this question. In this setup, the oracle graph is simply the adjacency graph of the continental
US naturally which will be used directly in our scanning procedure.
For this dataset, we have direct observations of node measures: the percentage of males and females who 
reported smoking or drinking heavily in each state. Using 
gender as the group, 
we fit standard linear models for each candidate subgraph, and compute the difference of gender-wise 
slopes statistic as described above. 
In Figure \ref{fig:tobalc}, we see the regions identified using our method, and interpret some of the tobacco usage findings here.

In the northeast, we see that women have reduced their tobacco usage at a significantly faster rate than men compared to the rest of the country. 
We suspect that this may be at least partly tied to the development of women's cigarette brands in the late 1960s and 1970s 
followed by subsequent aggressive public policy campaigns in the 1990s and 2000s to highlight health risks beyond 
pulmonary or cardiovascular diseases for women (e.g., infertility, reduced bone-density in post-menopausal women). We also see that 
state-wide indoor smoking bans were put in place in the Northeast 
ahead of many other states in the union.  
In the South, the trends among men and women also seems to differ significantly. (see Fig. \ref{fig:tobalc}). 
Apart from health factors, the group-wise differences in the group-wise trends may also be explained by 
a few reasons identified in a study in 2007 \cite{HEC:HEC1223} which found that as the state sales tax on cigarettes changed (increased), 
women were significantly more price elastic than men. Between 2006 and 2008, the cigarette tax increased dramatically for all of the 4 states identified \textit{except for Louisiana}, whose tax rate has remained constant. Additionally, while Arkansas did increase their cigarette tax in 2009, they did \textit{not increase taxes in locations near borders shared with higher taxing states}. These intricate relationships among states lend credibility to the fact that our scan statistics framework is indeed identifying interesting sub-regions, and suggests that the full covariance-trajectory pipeline may be more appropriate if effects beyond the means are relevant within an analysis. 

%

\section{Pipeline Evaluation on Simulations and Baby Name Trends Over Time}
\label{sec:pipeval}
We next evaluate the ability of our entire analysis 
pipeline to identify group differences across temporally evolving \textit{covariance} trajectories. 
In many existing analyses, the effect of the mean differences may be stronger than the effect of the interaction matrix. 
However, in cases where the {\em mean signal is weak}, we expect that the covariance effect will be important. To evaluate 
our model in this regime, we perform a set of simulation studies and also analyze a publicly available longitudinal dataset.

{\bf Simulations.} We randomly generate SPD matrices from a `path' of 4 discrete points along the manifold, and use these data 
as population 
covariance matrices to generate 0-mean sample data. Table \ref{recover-table} shows the results of the hypothesis testing 
procedure with 50 features averaged over 100 runs, where both the true number of features with covariance trajectory differences, $p_t$, 
and the number of samples per group, $n$, were varied. 
As expected, our recovery rate increases nicely 
as a function of the number of samples $n$ and decreases as the size of region of change $p_t$ is increased when $n$ is held constant.

\begin{table}[]
	{
		\caption{Detection Accuracy of hypothesis test scheme (100 runs).}
		\begin{center}
			{\begin{tabular}{lllll}
					\toprule
					\toprule
					& $p_t = 5$& $p_t = 8$	& $p_t = 10$&  $p_t = 15$ 	\\
					\midrule
					$n=10$	&  0.06  &  0.02  &  0.04  &  0.03 \\
					$n=20$	&  0.75  &  0.75  &  0.53  &  0.29 \\
					$n=50$	&  0.99  &  1.00  &  1.00  &  0.80 \\
					$n=100$	&  1.00  &  1.00  &  1.00  &  0.95 \\
					$n=200$	&  1.00  &  1.00  &  1.00  &  0.98 \\
					$n=1000$&  1.00  &  1.00  &  1.00  &  1.00 \\
					\bottomrule
					\bottomrule
				\end{tabular}}
			\end{center}
		}
		\label{recover-table}
\end{table}
	
We compare our model to baseline methods that may be used in practice for the foregoing 
group difference hypothesis test. In standard applications, general linear models (GLMs) are often the first line of attack. 
When the covariates are assumed to be independent, a simple linear model as in \eqref{eq:euclideanhyptest} may be suitable. 
However, when the group difference is influenced by specific interactions between covariates, such linear models require additional care. 
A typical solution is to introduce pairwise interaction terms into the model -- a choice between 
all possible interactions or \textit{specific interactions specified by an expert}. The first model has 
problems since the number of samples $n \ll p^2$. In the second model, we depend completely on the user's choice of interactions, 
and must correct for multiple testing when testing different models, at least partly reducing the power of the final test.
\begin{figure}[!]
	\begin{center}
		\includegraphics[width=\textwidth]{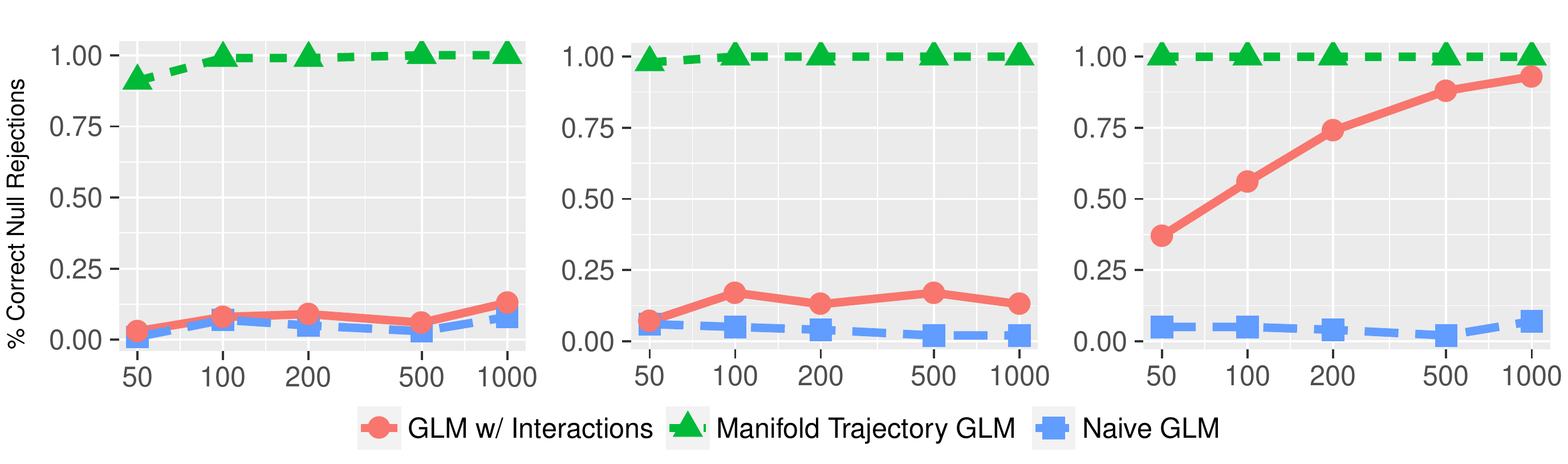}
		\caption{\label{fg:sim_graphs}{\footnotesize Correct null hypothesis rejections over 100 runs for three models. For $p = 50$ features, each plot shows the rejection rate
		for $p_t \in \{4, 8, 20\}$ (from left to right) respectively as a function of the number of sample points.} }
	\end{center}
\end{figure}
Figure \ref{fg:sim_graphs} shows the value of our method over these models. For the interaction GLM case, we randomly select interaction terms to include in the GLM, with size $p_t$ (the ground truth number of variables in the interaction). In this way, we approximate the effect of an oracle specifying to the GLM which terms may describe the underlying interaction. 
We report the fraction of significance tests where a significance threshold 
of $p \leq 0.05$ was found for each model, averaged over 100 runs. 
We see that our proposed scheme consistently achieves near-perfect results in terms of the percentage of null hypotheses that 
were correctly rejected (i.e., there was a significant group-difference signal). The power of scan statistics on 
graphs is particularly evident in the needle in haystack setting where the true differential signal is small ($p_t \leq 8$)
and the sample size is small to medium. When the sample size is large and $p_t$ is also large, the standard linear model with additional interaction terms starts to approach the statistical performance of our algorithm.

{\bf Longitudinal trends in Baby Names.} In addition to the simulations above, we report results from a simple analysis of 
how male/female baby names evolve over time over the last century. 
The United States Social Security Administration provides a publicly available dataset listing the frequency of the 
top 1000 baby names in each state for the last 106 years.
We evaluate our model in this context to examine which ``sub-group'' of states tend to evolve (or change) in their
``name agreement" (or correlation) over time between boy names and girl names.
Here, rather than calculating a sample covariance
at each timepoint, we calculate a rank correlation matrix instead. 
For example, if two neighboring Gulf Coast states, say Georgia and Alabama, substantially 
agreed on both boys and girls names in the period following the second World War, but gradually this agreement 
declined over time for girls (but not boys), 
we expect that our scan statistics on graphs hypothesis test 
will segment out this differential signal (in slope trends) from the planar graph induced by the states sharing 
a border.
Shown in Figure \ref{fig:usmap} are the regions identified using our method, applied on only the rank correlations for the top 10 names for both genders per state per year. Each highlighted region indicates a sub-group
where their ``trends of correlation (or agreement/disagreement)'' in preferred baby names over the last century 
varies between boys and girls. 
For states not identified by our model (in gray), we can conclude that
the state-to-state name preference-interactions may have still 
evolved over time but we have insufficient statistical 
evidence to conclude that such trends (slopes) are different between boys and girls. 

\begin{figure}[]
	\begin{center}
		\includegraphics[,trim={5cm 5cm 5cm 5cm}, clip, width=0.7\textwidth]{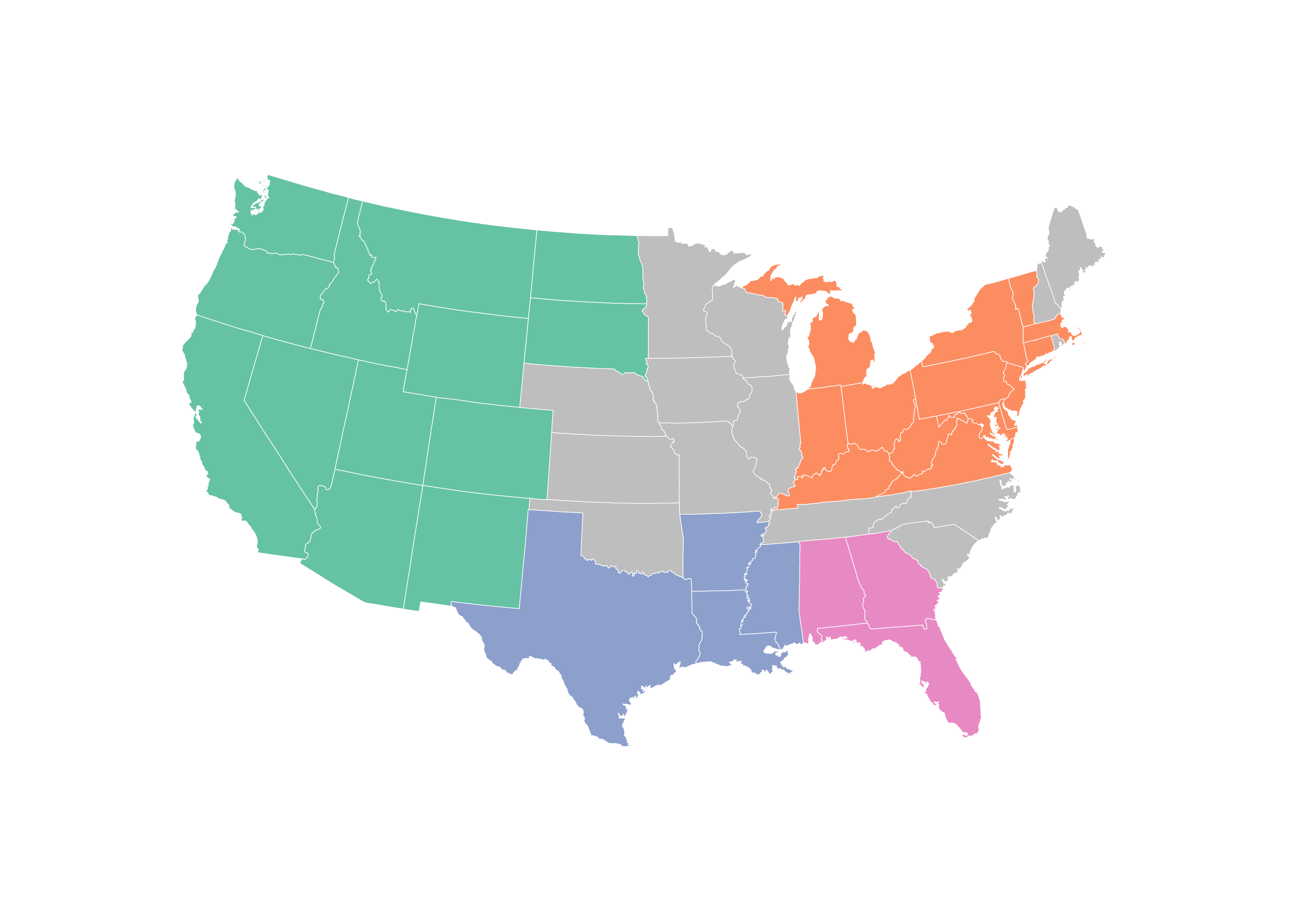}
		\caption{\label{fig:usmap} Contiguous states identified as having significantly different time-varying co-occurrences between boys and girls baby names from 1910 to 2015. Best viewed in color.}
	\end{center}
\end{figure}


\section{Identifying Differentially Covarying Features in Preclinical Alzheimer's Disease}
\label{sec:wrap}
We now describe experiments and results focused on the key motivation of this work --- to facilitate 
analysis of a longitudinal study of individuals at risk for Alzheimer's disease (AD) where the statistical 
signal is weak (with small to medium sample sizes). We describe the dataset details followed by the analysis and then interpret our conclusions 
in the context of scientific results that have been published in the literature in aging and dementia.

{\bf Study background.} We analyzed data from a 
cohort of individuals who have been longitudinally tracked for at least three visits over multiple years, as part of an ongoing study (since 2001)
to understand the disease processes in the brain {\em before} an individual exhibits signs of 
cognitive decline due to Alzheimer's Disease (AD) \cite{sager2005middle}. The study, Wisconsin Registry for 
Alzheimer's Prevention (WRAP)  
is among the largest of its kind in existence, focused on ``preclinical'' AD, i.e., when the 
individuals are still cognitively healthy, offering a window into the early disease processes 
where treatments, drugs and interventions are likely to be most effective. 
WRAP and its ancillary studies 
acquire neuroimaging data (MRI, PET with different tracers, diffusion MRI) and various clinical test scores, 
genetic and demographic data as well as clinical measures such as Cerebrospinal Fluid (CSF). Our analysis 
seeks to understand subtle group-wise differences in longitudinal patterns of dependencies between these measures at this early stage of the disease. 

{\bf Dataset.} The dataset consisted of 114 subjects with imaging data from at least two types of imaging modalities: Positron emission tomography and 
diffusion weighted Magnetic Resonance (MR) images. 
Positron emission tomography (PET) images were used to calculate, using well-validated pre-processing pipelines, 
the mean amyloid-plaque load (an important biomarker for AD) in 16 different anatomical regions of interest in the brain. 
Amyloid plaque is known to be an AD-related pathology and generally {\em precedes} onset of cognitive symptoms. 
Separately, diffusion tensor MR imaging (DTI) data were processed and used to calculate both Fractional Anisotropy (FA) and Mean Diffusivity (MD) in 48 distinct regions \cite{mori2008stereotaxic}. 
DTI images provide information about structural connectivity between gray matter regions in the brain. 
In addition to these $108$ ($48 \times 2 + 16$) image-derived features, 
we also included in the analysis the participant's scores on a battery of cognitive tests, known to be correlated with various neuropsychological functions \cite{lezak2004neuropsychological}. 
Differences were evaluated on various groupings of the subjects which were, for the most part, 
based on known results in the literature.   
Specifically, gender, APOE (Apolipoprotein E) genotype and amyloid positivity (based on thresholding the amyloid plaque summaries) have 
all been evaluated as significant in
AD studies \cite{racine2014associations} but often such analyses involve a population covering a broader disease spectrum where 
the signal is much stronger. 
%

{\bf Is analysis of second order statistics necessary?} In Figure \ref{fig:imagehists}, we present histograms detailing the distribution of two critical cognitive tests, stratified across various groups of scientific interest. Evaluating these distributions were the key motivation for our exploration into the methods described in the paper. Small differences in means across groups {\em regardless of grouping selection (i.e., stratification variable)}, and the saturation that occurs at the ceiling of cognitive test scores and other preliminary experiments conducted by us suggest that standard analyses are not sensitive enough to identify subtle higher-order differences.
\begin{figure}[]
	\centering
	\includegraphics[width=0.99\textwidth]{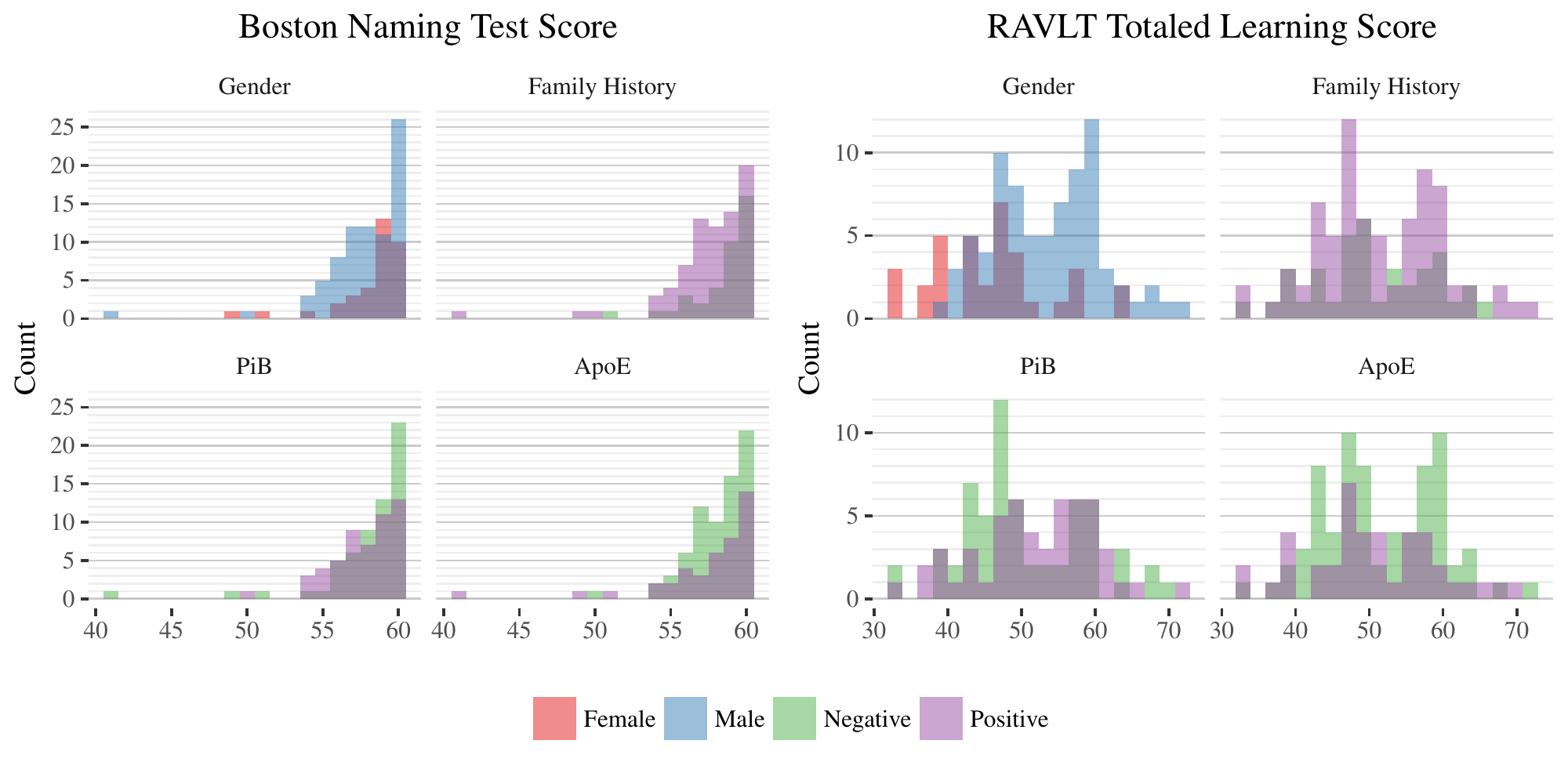}
	\caption{Histograms of the Boston Naming Test Scores and RAVLT Total Scores for all time points for the $114$ individual measurements across different group separations. The means for each test score is not significantly different across different stratification variable.}
	\label{fig:imagehists}
\end{figure}
\subsection{Results for Group difference analysis for individuals with imaging data}
We now describe, one by one, the components of the largest feature subset discovered for each stratification scheme 
and highlight the main scientific findings. In most cases, we provide a brief scientific interpretation of the results
for the interested reader. 
Additional details and results are available in the appendix.
%

{\bf A) Graph Scan Statistics on slope differences across gender.} 
The most significant (based on region-score) subset identified by the gender grouping was between the FA DTI measurement in the left cingulum gyrus 
as well as the scores on the Rey Auditory Verbal Learning Test (RAVLT). In recent AD research, gender has been identified as a factor in the progression 
of various pathology measures (e.g., incidence and prevalence of AD is higher in women \cite{fratiglioni1991prevalence,rimol2010sex}), and has contributed to a formal NIH notice (NOT-OD-15-102). However, we note that previous work in the field has \textit{not} identified gender-related 
differences when looking {\em only} at diffusion measures in the cingulum \cite{lin2014cingulum}. Our algorithm 
successfully identified longitudinal 
changes in {\em interaction} between these variables which supports the earlier results, and provides some evidence that as men and women age, 
their cognitive decline as measured by RAVLT manifests differently in relation to the cingulum gyrus.

\begin{table}
	\scriptsize
	\centering
	\begin{tabular}[t]{ll}
		\toprule
		\multicolumn{2}{c}{\textbf{Gender}}\\ \midrule \midrule
		
		Set 1     & RAVLT Total (1-5) \\ & FA Cingulum  L	\\
		\midrule
		Set 2 & FA Medial lemniscus L	\\ & FA Cingulum (hippocampus) L		\\
		& FA Post thalamic radiation L \\ \midrule
		Set 3    & FA Corticospinal tract R \\ & FA Superior O.F. fasciculus  R \\	\midrule\bottomrule
	\end{tabular}
	\hfill
	\begin{tabular}[t]{ll}
		\toprule
		\multicolumn{2}{c}{\textbf{Genotype: APOE4}}\\ \midrule \midrule
		Digit Span Backward Raw Score & Stroop Color-word Score \\ 
		PiB Cingulum Post L &    PiB Cingulum Post R         \\
		PiB Frontal Med Orb L & 		PiB Frontal Med Orb R \\ 
		PiB Precuneus L & PiB Precuneus R \\ 
		PiB SupraMarginal &  PiB Temporal Mid R \\ \midrule		 \bottomrule
	\end{tabular}
	\caption{Group difference across Gender (left) and Genotype APOE4 expression (right). Three disjoint sets of features were identified as coavarying significantly differently among gender, while one larger set was identified in the genotype stratification.}
	\label{tab:wrapIMG}
\end{table}

{\bf B) Graph Scan Statistics on slope differences across genotype.} 
Next, we stratified the cohort
 based on the genotype known to be most closely linked with AD, i.e., the APOE (Apolipoprotein E) gene \cite{corder1993gene} --- 
we inherit one APOE allele from each parent; having one or two copies of the e4 allele increases a person's risk of getting AD 
whereas the rarer e2 allele is associated with a lower risk of AD. 
Using this stratification, we obtain a low-risk and an at-risk group of individuals. 
Here, we identified amyloid-load regions within the medial and lateral parietal lobes 
and find that in the ``low-risk" group, the covariances between Digit Span and Stroop Color-Word scores 
(attention and concentration scores) and amyloid load moves from strongly negative towards $0$ as a function of age (Table \ref{tab:wrapIMG}). 
In the ``at-risk'' group 
(APOE4), however, we find that as a function of age, the features become more and more positively correlated. 
Existing studies have shown that the accumulation 
of amyloid is significantly different across APOE4 gene expression \cite{mormino2014amyloid}, and our results provide some evidence 
that the expression of the genotype may interact with cognitive scores as well, {\em even at this early stage of the disease}, 
when the individuals in our cohort are 
cognitively healthy. The sets of features showing a differential signal are presented in Table \ref{tab:wrapIMG}.
%

\begin{table}[]
	\small
	\centering
	\begin{tabular}{lll}
		\toprule
		\multicolumn{3}{c}{\textbf{Amyloid Load (PiB Positivity)}}\\ \midrule \midrule
		Set 1 & PiB Angular L/R & PiB Cingulum Ant L/R \\
		& PiB Cingulum Post L/R & PiB Frontal Med Orb L/R \\
		& PiB Precuneus L/R & PiB Temporal Sup L/R \\
		& PiB Temporal Mid L/R & \textbf{PiB SupraMarginal L} \\
		\midrule
		Set 2     & FA Cerebral peduncle R   & FA Cerebral peduncle L	\\
		& MD Corticospinal tract R	& MD Corticospinal tract L		\\
		& Trail-Making Test Part A Score  & MD Cerebral peduncle R \\ 
		&PET Cingulum Post R  &  \\ \midrule\bottomrule
	\end{tabular}
	\caption{Group difference across Amyloid Load (PiB Positivity)}
	\label{tab:wrapPIB}
\end{table}

{\bf C) Graph Scan Statistics on slope differences across amyloid load positivity.} 
As briefly described above, amyloid load is an important biomarker for AD. For our analysis, amyloid (or PiB) positivity 
is calculated using the mean amyloid PiB measures across all brain regions using a PiB PET image scan of the participant. 
When we used this measure for stratification (threshold was set at $1.18$, following \cite{darst2017pathway}), 
our model identified fifteen of the sixteen PiB regions that were input to the model when the density of the oracle graph was set to be high. 
This result is as expected, but interestingly we find that controlling for the linear combination of the features (through centering), 
the residual error \textit{still} has significant signal with the PiB positivity measure, indicating that amyloid burden \textit{interactions} 
across brain regions plays a very important role in AD progression \cite{hardy2002amyloid,hardy1992alzheimer,tanzi2005twenty,jack2010brain}. When the sparsity of the oracle graph was increased, however, four neighboring regions, the left and right corticospinal tract and the left and right cerebral peduncle were identified on both PiB and DTI measures (supported by the literature \cite{douaud2011dti}), together with Part A of the Trail Making Test (see Table \ref{tab:wrapPIB}) which 
happens to be used in AD diagnosis \cite{albert2011diagnosis}. 
This suggests that changes in atrophy within these regions, as measured by DTI, co-occur with changes in amyloid burden. Additionally, because these regions are highly correlated with rough and fine motor ability \cite{naidich2009duvernoy}, it seems plausible that amyloid positivity will lead to higher `covariation' in the regions associated with 
a measure of fine motor speed, i.e., the Trail Making Test.

\subsection{Results for for Group difference analysis for individuals with Cognitive Testing data}

In addition to the dataset presented above, we apply our method to a much larger dataset consisting of approximately 1500 individuals with only cognitive testing data collected in a longitudinal manner. Each individual was administered these tests for between two and three time-points, yielding approximately $n = 4000$ samples for our model. For each assessment, a conference of experts applied a diagnostic label indicating normal cognition or mild cognitive impairment. Using this binary classification, we can stratify our population for group difference analysis.
We find that among many different significant subsets, the covariance trajectory among the scores on both parts of the Trail-Making Test and 
on all trials of the RAVLT test explain a significant group difference. These have previously been shown to be the {\em most sensitive tests} for 
early cognitive decline \cite{albert2001preclinical}. 
Table \ref{tab:wrapCC} displays the other tests identified by our algorithm, and additional experiments on this larger cohort 
can be found in the appendix. 

\begin{table}
	\small
	\centering
	\begin{tabular}{ll}
		\toprule
		\multicolumn{2}{c}{\textbf{Expert Consensus Diagnosis}}\\ \midrule \midrule
		WAIS-3 LNS Raw Score &
		Boston Naming Test Total Score \\
		RAVLT A2 Raw Score &
		RAVLT A3 Raw Score \\
		RAVLT A4 Raw Score &
		RAVLT A5 Raw Score \\
		RAVLT A6 Raw Score &
		RAVLT Delayed Recall Raw Score \\
		Trail-Making Test Part A &
		Trail-Making Test Part B \\
		Clock Drawing Test Score &
		CES Depression Scale Score \\
		\bottomrule
		\bottomrule
	\end{tabular}
	\caption{Group difference localization across expert clinical diagnosis. With significantly more samples and a larger set of cognitive tests, those above were identified as significantly different across the expert consensus measure.}
	\label{tab:wrapCC}
\end{table}

\subsection{Baseline.}

In various experiments on this dataset, when the MMGLM procedure is performed for the entire feature set in totality ({\em not} 
utilizing any of the proposed ideas based on scan statistics), 
and the null distribution derived using permutation testing, the procedure {\em yields no significance across \textit{any} 
scientifically interesting group stratifications}. 
This implies that the ability to search over different blocks of the covariance matrix is critical in identifying meaningful group differences 
in the trajectories, unavailable 
using alternate schemes. For instance, simpler strategies work well enough for datasets such as ADNI -- 
  which includes diseased subjects as well as controls -- 
  where the signal is stronger and even temporal modeling may be unnecessary.
  While the scientific results need to be interpreted with caution and reproducibility experiments on 
other similar datasets (both within the US and internationally) are in the planning phase, 
we believe that the ability to localize 
differences in these interaction patterns in a statistically rigorous manner is valuable and these findings can be investigated standalone, via 
more classical schemes (e.g., structural equation modeling).

\section{Conclusions}
The analysis of datasets to identify where clinically disparate groups differ is pervasive in biology, neuroscience, genomics and epidemiological studies. 
We find that graphical models are an ideal tool to analyze high-dimensional data in these areas but have been sparingly used for the analysis of 
group-wise differences, especially in a longitudinal setting. 
Motivated by an application related to longitudinal analysis of imaging and clinical/cognitive data from otherwise healthy individuals 
who are at risk for Alzheimer's disease (AD), we show how a combination of manifold regression with a generalization of scan statistics to the graph setting yields 
tools that can be directly deployed. 
We present an efficient algorithm and develop the theoretical results showing the regimes where its application is appropriate. 
In various experiments, while the standard schemes are not sufficiently powered to detect the signal, our proposed formulation is able to 
detect meaningful group difference patterns, many of which have a clear scientific interpretation. 
We believe that these results are promising for the neuroimaging application 
described and other regimes where group-wise analysis is desired but the number of features is large.


\acks{ This research was supported in part by NIH grants R01 AG040396, AG021155, EB022883 
and NSF grants DMS 1265202 and CAREER award 1252725. The authors were also supported by 
the \href{http://cpcp.wisc.edu/}{UW Center for Predictive Computational Phenotyping} (via BD2K award AI117924) and the 
\href{http://www.adrc.wisc.edu/}{Wisconsin Alzheimer's Disease Research Center} (AG033514). 
Mehta was supported by a fellowship via training grant award T32LM012413. 
}

\newpage
\appendix
\section{Technical Proofs.}
\subsection{Proof of Lemma \ref{lm:entropy}}
To remind the reader, this result was necessary in order to allow us to reduce the number of subgraphs (regions) that need to be evaluated over the graph. By bounding the covering number we have a guarantee that we do not need to consider an exponential number of subgraphs in order to find a localization.

%
%

\begin{proof}
	To upper bound $N(A,\epsilon)$, we first construct the $\epsilon$-covering set of $\Rcal(A)$ under metric $d$.
	To this end, we decompose $\Rcal(A)$ into several disjoint sets 
	$$
	\Rcal_j(A)=\left\{B(v,r)\in \Rcal(A): \left(1-{(j+1)\epsilon \over 2}\right)A<|E(B(v,r))|\le \left(1-{j\epsilon \over 2}\right)A\right\} ,
	$$
	for $j=0,1,\ldots, \lceil{1 \over \epsilon}\rceil$. Our strategy is to construct $\epsilon$-covering set for each set $\Rcal_j(A)$. 
	
	We only construct $\epsilon$-covering set for $\Rcal_0(A)$; $\Rcal_j(A)$ ($j\ge 1$) can be treated similarly.
	To construct the $\epsilon$-covering set for $\Rcal_0(A)$, we denote by $d_{v,r}$ the largest positive number such that
	\begin{equation}
	\label{eq:d}
	{|E(B(v,r-d_{v,r}))| \over |E(B(v,r))|}\ge 1-{\epsilon\over 2},
	\end{equation}
	for every $v\in V$ and $r\in \NN$. Let $\Dcal_{1}$ the collection of $d_{v,r}$ such that $B(v,r)\in \Rcal_0(A)$, i.e.
	$$
	\Dcal_{1}=\{d_{v,r}:B(v,r)\in \Rcal_0(A)\},
	$$
	and $\Vcal_{1}$ the collection of nodes such that $B(v,r)\in \Rcal_0(A)$, i.e. 
	$$
	\Vcal_1=\{v:B(v,r)\in \Rcal_0(A)\}.
	$$
	We pick up the largest number in $\Dcal_1$, denoted by $d_{v_1,r_1}$, i.e. $
	d_{v_1,r_1}\ge d_{v,r}\ \forall\ d_{v,r}\in \Dcal_1$
	and define $\tilde{\Vcal}_{1}$  as
	$$
	\tilde{\Vcal}_{1}=\{v\in \Vcal_1:v\in B(v_1,d_{v_1,r_1}/2)\}.
	$$
	After defining $\tilde{\Vcal}_{1}$, $\Dcal_{2}$ and $\Vcal_2$ can be defined as
	$$
	\Dcal_{2}=\Dcal_{1}\setminus \{d_{v,r}:v\in \tilde{\Vcal}_{1}\}\qquad{\rm and}\qquad \Vcal_{2}=\Vcal_{1}\setminus  \tilde{\Vcal}_{1}.
	$$
	Then we can pick up the largest number in $\Dcal_{2}$, denote by $d_{v_2,r_2}$ and $\tilde{\Vcal}_{2}$ can be defined similarly.
	We can repeat the above process until $\Dcal_M$ and $\Vcal_M$ are empty for some $M$.We actually obtain a partition of $\Vcal_1$, 
	$$
	\bigcup_{i=1}^{M}\tilde{\Vcal}_i=\Vcal_1\qquad{\rm and}\qquad \tilde{\Vcal}_{i_1}\cap \tilde{\Vcal}_{i_2}=\emptyset \qquad 1\le i_1< i_2\le M.
	$$
	Based on $d_{v_1,r_1},\ldots, d_{v_M,r_M}$, we are ready to prove the set
	$$
	\Rcal_0(A,\epsilon)=\{B(v_i,r_i):1\le i\le M\}
	$$
	is actually an $\epsilon$-covering set for $\Rcal_0(A)$. To this end, it is equivalent to show that for arbitrary $B(v',r')\in \Rcal_0(A)$, we have
	\begin{equation}
	\label{eq:dist}
	d(B(v',r'),B(v_i,r_i))\le \epsilon
	\end{equation}
	when $v'\in \tilde{\Vcal}_i$. To show (\ref{eq:dist}), we consider two cases where $r'>r_i-d_{v_i,r_i}/2$ and $r'\le r_i-d_{v_i,r_i}/2$.
	When $r'>r_i-d_{v_i,r_i}/2$, then 
	$$
	B(v_i,r_i-d_{v_i,r_i})\subset B(v',r').
	$$
	Combining above result, (\ref{eq:d}), and the definition of $\Rcal_0(A)$ yields
	\begin{align*}
	&{|E(B(v',r'))\cap E(B(v_i,r_i)) |\over\sqrt{|E(B(v',r'))||E(B(v_i,r_i))|}}\\
	\ge& {|E(B(v_i,r_i-d_{v_i,r_i})) |\over\sqrt{|E(B(v',r'))||E(B(v_i,r_i))|}} \\
	\ge& \sqrt{1-{\epsilon\over 2}}{|E(B(v_i,r_i-d_{r_i}))| \over |E(B(v',r'))|}\\
	\ge& 1-\epsilon.
	\end{align*}
	On the other hand, if $r'\le r_i-d_{v_i,r_i}/2$, then 
	\begin{align}
	B(v',r')\subset B(v_i,r_i).
	\end{align}
	By definition of $\Rcal_0(A)$, we can get
	$$
	{|E(B(v',r'))\cap E(B(v_i,r_i)) |\over\sqrt{|E(B(v',r'))||E(B(v_i,r_i))|}}\ge \sqrt{|E(B(v',r'))| \over |E(B(v_i,r_i))|}\ge  1-\epsilon.
	$$
	Therefore, (\ref{eq:dist}) is proved and $\Rcal_0(A,\epsilon)$ is an  $\epsilon$-covering set for $\Rcal_0(A)$.
	
	The rest of the proof is to bound the cardinality of $\Rcal_0(A,\epsilon)$, i.e. $M$.
	Note that (\ref{eq:lightskirt}) implies there exists some constant $D_{H,S}$ only depending on $H$ and $S$ such that, for any $v\in V$ and $r\in \NN$,
	$$
	|E(B(v,r/2))|\ge D_{H,S} |E(B(v,r))|.
	$$
	By the definition of $d_{v_i, r_i}$, we can ensure $B(v_i,d_{v_i,r_i}/4)$ are disjoint.
	Hence, this implies
	$$
	|E(\tilde{\Vcal}_i) |\ge |E(B(v_i,d_{v_i,r_i}/4))|\ge D_{H,S}^2|E(B(v_i,d_{v_i,r_i}))|\ge D_{H,S}^2 HA\epsilon^S/2^{S+1}.
	$$
	The last inequality is suggested by (\ref{eq:lightskirt}) and (\ref{eq:d}). The volume argument yields
	$$
	M\le {|E|\over D_{H,S}^2 HA\epsilon^S/2^{S+1}}\le {2^{S+1}\over D_{H,S}^2H}{|E|\over A}\left({1\over \epsilon}\right)^S
	$$
	(\ref{eq:entropy}) is obtained upon application of the above to each $\Rcal_j(A)$.
\end{proof}

\subsection{Proof of Theorem \ref{lm:mainthm}}
	
Before we are ready to prove Theorem \ref{lm:mainthm}, we need the following result:

\begin{lemma}
	\label{lm:diffbd}
	Let $Y_1,\ldots,Y_d$ be i.i.d. standard Gaussian variable, i.e. $N(0,1)$ and $a_1,\ldots,a_d$ be a sequence of numbers.
	If 
	\begin{equation}
	Z=\sum_{i=1}^d a_i(Y_i^2-1),
	\end{equation}
	then
	\begin{equation}
	\PP(|Z|\ge 2|a|_2\sqrt{x}+2|a|_\infty x)\le 2\exp(-x)
	\end{equation}
	where $|a|_2=\sqrt{\sum_{i=1}^d a_i^2}$ and $|a|_\infty=\max_{i=1,\ldots,d}|a_i|$.
\end{lemma}

\begin{proof}
	This is a direct extension of lemma 1 in \cite{laurent2000adaptive} to the negative case. We follow arguments similar to theirs.
	Let $\phi(x)$ be the  the logarithm of the Laplace transform of $Y_i^2-1$. For any $-1/2<x<1/2$,
	$$
	\phi(x)=\log\left(\EE\left(\exp(x(Y_i^2-1))\right)\right)=-x-{1\over 2}\log(1-2x)\le {x^2\over 1-2|x|}.
	$$
	This leads to
	\begin{align*}
	\log(\EE(e^{xZ}))&=\sum_{i=1}^d \log\left(\EE\left(\exp(a_ix(Y_i^2-1))\right)\right)\\
	&\le \sum_{i=1}^d {a_i^2x^2\over 1-2|a_i|x}\\
	&\le {|a|_2^2x^2\over 1-2|a|_\infty x}
	\end{align*}
	With the same arguments in \cite{laurent2000adaptive}, we could prove that 
	$$
	\PP\left(Z\ge 2|a|_\infty x+2|a|_2\sqrt{x}\right)\le \exp(-x).
	$$
	The other direction can be proved if we apply the same argument for $-Z$.
\end{proof}

With this in hand we proceed to prove Theorem \ref{lm:mainthm}.


\begin{proof}
In the following proof, $C$ always refers to some constant, although its value may change from place to place. First, we prove (\ref{eq:criticalvl}). To this end, we prove concentration inequalities for $L_R$ for some $R$ and $L_{R_1}-L_{R_2}$ for some $R_1\ne R_2$.
Since we assume the noise follows normal distribution, we have 
$$
(\hat{\Bbeta}_1^R-\hat{\Bbeta}_2^R)^T\Sigma_R^{-1}(\hat{\Bbeta}_1^R-\hat{\Bbeta}_2^R)={\sum X_i^2-(\sum X_i)^2\over 2}\|\hat{\Bbeta}_1^R-\hat{\Bbeta}_2^R\|^2\sim \chi^2_{|E(R)|}.
$$
By tail bound for $\chi^2$ random variables (see e.g. \cite{laurent2000adaptive}), we can yield
\begin{equation}
\label{eq:single}
\PP\left(L_R>2t+{2t^2\over \sqrt{|E(R)|}}\right)\le \exp(-t^2).
\end{equation}
By definition, $L_{R_1}-L_{R_2}$ can be written as
$$
L_{R_1}-L_{R_2}={\sum_{i\in R_1\setminus R_2}Z_i\over \sqrt{|E(R_1)|}}+ \left({1 \over \sqrt{|E(R_1)|}}-{1 \over \sqrt{|E(R_2)|}}\right)\sum_{i\in R_1\cap R_2}Z_i -{\sum_{i\in R_2\setminus R_1}Z_i\over \sqrt{|E(R_2)|}}
$$
where $Z_i$ are independent random variable following distribution $\chi_1^2-1$.
Lemma~\ref{lm:diffbd} implies
\begin{equation}
\label{eq:diff}
\PP\left(|L_{R_1}-L_{R_2}|>2\sqrt{2d(R_1,R_2)}t+{2t^2\over \min(|E(R_1)|,|E(R_2)|)}\right)\le 2\exp(-t^2).
\end{equation}

We now proceed to prove (\ref{eq:criticalvl}) by applying a chaining argument (See \cite{talagrand2006generic}) and concentration inequalities (\ref{eq:single}) and (\ref{eq:diff}).
Recall $\Rcal_{app}(A,\epsilon)$ is the smallest $\epsilon$-covering set of $\Rcal(A)$ and $N(A,\epsilon)$ is the covering number of $\Rcal(A)$.
For any subgraph candidate $R$, we denote by
$$
\pi_l(R)={\arg\min}_{R'\in \Rcal_{app}(A,e^{-l})} d(R,R').
$$
For any $l^\ast>l_\ast$, which will be specified later, we write $\max_{R\in\Rcal(A)}L_R$ into three parts
$$
\max_{R\in\Rcal(A)}L_R\le \max_{R\in\Rcal(A)}|L_R-L_{\pi_{l^\ast}(R)}|+\sum_{l=l_\ast}^{l^\ast-1}\max_{R\in\Rcal(A)}|L_{\pi_{l+1}(R)}-L_{\pi_{l}(R)}|+\max_{R\in\Rcal(A)}L_{\pi_{l_\ast}(R)}.
$$
Now, we bound these three terms above separately.
\begin{enumerate}
\item[] \textbf{Term 1}. Let $l^\ast=2\log |E|$. By concentration inequality (\ref{eq:diff}) and union bound, we have 
\begin{align*}
& \PP\left(\max_{R\in\Rcal(A)}|L_R-L_{\pi_{l^\ast}(R)}|>{2\sqrt{2(x+\log|E|)}\over |E|}+{4x+8\log |E|\over A}\right)\\
\le & |\Rcal(A)| \PP\left(|L_R-L_{\pi_{l^\ast}(R)}|>{2\sqrt{2(x+\log|E|)}\over |E|}+{4x+8\log |E|\over A}\right)\\
\le & 2{|\Rcal(A)|\over |E|^2}\exp(-x)\le 2\exp(-x) 
\end{align*}
for $x<\log |E|$. Therefore, we have
$$
\PP\left(\max_{R\in\Rcal(A)}|L_R-L_{\pi_{l^\ast}(R)}|>{C(x+\log |E|)\over A}\right)\le \exp(-x),
$$
for $x<\log |E|$.
\item[] \textbf{Term 2}. Let $l_\ast=\log\log (|E|/A)$. 
Recall that the Avocado assumption (\ref{eq:lightskirt}) suggests that
\begin{equation}
\label{eq:cvnum}
N(A,\epsilon)\le C_{H,S}{|E|\over A}\left(1\over \epsilon\right)^{S+1}.
\end{equation}
Applying concentration inequality (\ref{eq:single}) along with
\begin{equation}
t=\sqrt{\log\left(|E|\over A\right)+(S+1)\log\log\left(|E|\over A\right)+x+C}
\end{equation}
and the union bound, we have 
\begin{align*}
&\PP\left(\max_{R\in\Rcal(A)}L_{\pi_{l_\ast}(R)}>2t+{2t^2\over \sqrt{A}}\right)\\
\le & N\left(A,{1\over \log (|E|/A)}\right) \PP\left(L_{\pi_{l_\ast}(R)}>2t+{2t^2\over \sqrt{A}}\right)\\
\le & C_{H,S}{|E|\over A}\left(\log {|E|\over A}\right)^{S+1} \PP\left(L_{\pi_{l_\ast}(R)}>2t+{2t^2\over \sqrt{A}}\right)\\
\le & \exp(-x) 
\end{align*}
for $x<\log |E|$. Here we also apply condition (\ref{eq:setsize}). Therefore, we obtain
$$
\PP\left(\max_{R\in\Rcal(A)}L_{\pi_{l_\ast}(R)}>2\sqrt{\log\left(|E|\over A\right)+(S+1)\log\log\left(|E|\over A\right)+x}+C\right)\le \exp(-x) 
$$
for $x<\log |E|$.
\item[] \textbf{Term 3}. For any given $l$, application of concentration inequality (\ref{eq:diff}), covering number condition (\ref{eq:cvnum}), and the union bound yields,
\begin{align*}
&\PP\left(\max_{R\in\Rcal(A)}|L_{\pi_{l+1}(R)}-L_{\pi_{l}(R)}|>\sqrt{C(\log(|E|/A)+l+x)\over e^l}+{C(\log(|E|/A)+l+x) \over A }\right) \\
\le & C_{H,S}{|E|\over A}e^{(l+1)(S+1)}\PP\left(|L_{\pi_{l+1}(R)}-L_{\pi_{l}(R)}|>\sqrt{C(\log(|E|/A)+l+x)\over e^l}+{C(\log(|E|/A)+l+x) \over A }\right)\\
 \le & {\exp(-x)\over l^2}.
\end{align*}
for any $x<\log |E|$. 
With another standard application of the union bound, we have
\begin{align*}
&\PP\left(\sum_{l=l_\ast}^{l^\ast-1}\max_{R\in\Rcal(A)}|L_{\pi_{l+1}(R)}-L_{\pi_{l}(R)}|>\sqrt{C(\log(|E|/A)+x) \over \log(|E|/A)}+{\log^2|E|+x\log|E|\over A}\right)  \\ 
\le & \sum_{l=l_\ast}^{l^\ast-1} \PP\left(\max_{R\in\Rcal(A)}|L_{\pi_{l+1}(R)}-L_{\pi_{l}(R)}|>\sqrt{C(\log(|E|/A)+l+x)\over e^l}+{C(\log(|E|/A)+l+x) \over A }\right)\\
\le & \sum_{l=l_\ast}^{l^\ast-1} {\exp(-x)\over l^2}\\
\le & 2\exp(-x).
\end{align*}
\end{enumerate}
Putting the three terms above together yields
$$
\PP\left(\max_{R\in\Rcal(A)}L_R>2\sqrt{\log\left(|E|\over A\right)}+C(x+1)\right)\le {4\over \log(e|E|/A)}\exp(-x),
$$
where we apply $A\gg \log^2|E|$ and the inequalities $\sqrt{a+b}\le \sqrt{a}+\sqrt{b}$ and $\sqrt{a+b}\le \sqrt{a}+b/\sqrt{a}$.

Now, we apply this bound to $A=|E|2^{-k}$, $k\ge 0$ yielding 
$$
\PP\left(\max_{R\in\Rcal}\left(L_R-2\sqrt{\log{|E|\over |E(R)|}}\right)>C(x+1)\right)\le 8\exp(-x).
$$
This immediately suggests that $q_\alpha=O(1)$.

Now, let's turn to the case when a subgraph is significant, that is to prove (\ref{eq:power}). Assume the significant region is $R_0$. Using standard statistics we calculate the mean and variance of $L_{R_0}$
$$
\EE(L_{R_0})={(\Bbeta_1^{R_0}-\Bbeta_2^{R_0})^T\Sigma_{R_0}^{-1}(\Bbeta_1^{R_0}-\Bbeta_2^{R_0})\over \sqrt{|E({R_0})|}} \ \ \text{and} \ \  Var(L_{R_0})=2+4{(\Bbeta_1^{R_0}-\Bbeta_2^{R_0})^T\Sigma_{R_0}^{-1}(\Bbeta_1^{R_0}-\Bbeta_2^{R_0})\over |E({R_0})|}.
$$
By Chebyshev's inequality, we have
\begin{equation}
\label{eq:alche}
\PP\left({|L_{R_0}-\EE(L_{R_0})|\over \sqrt{Var(L_{R_0})}}>x\right)\le {1\over x^2}.
\end{equation}
If $(\Bbeta_1^{R_0}-\Bbeta_2^{R_0})^T\Sigma_{R_0}^{-1}(\Bbeta_1^{R_0}-\Bbeta_2^{R_0})\ge |E({R_0})|$, then (\ref{eq:alche}) suggests 
$$
\PP(L_{R_0}>\sqrt{|E(R_0)|})\to 1,\  |E|\to \infty
$$ by taking $x$ as a sequence (e.g., $\log\log(|E(R_0)|)$) which increases slow enough in (\ref{eq:alche}). This leads to (\ref{eq:power}).
If $(\Bbeta_1^{R_0}-\Bbeta_2^{R_0})^T\Sigma_{R_0}^{-1}(\Bbeta_1^{R_0}-\Bbeta_2^{R_0})<|E({R_0})|$, then $Var(L_{R_0})<6$. Then (\ref{eq:sigcond}) and (\ref{eq:alche}) imply
$$
\PP\left(L_{R_0}-2\sqrt{|E|\over |E(R_0)|}>q_\alpha\right)\to 1, \qquad{\rm as}\  |E|\to \infty.
$$
\end{proof}

\section{Implementation Details.}
The workflow below describes one run of our model given a sparsity is specified for the oracle graph procedure.
\begin{enumerate}
\item \textbf{Oracle Graph.} As noted in the main paper, we use \textit{graphical lasso (glasso)} to generate an \textit{oracle graph}, which allows to define structured regions (subgraphs) for scan statistics on graphs. 
Each element of the input matrix $C$ in \eqref{eq:glasso} for glasso is generated by  calculating the slope for each position of the covariance matrix across the predictors for each group, and then taking the difference between the groups.
The following inverse covariance estimation problem, \textit{glasso}, is then solved using existing MATLAB interfaces to fast C implementations.
\begin{equation}
\Theta = \arg\min_{\Theta \succeq 0} \quad -\log|\Theta| + tr(C\Theta) + \lambda||\Theta||_{1}
\label{eq:glasso}
\end{equation}

With sparsity parameter $\lambda$, this procedure generates a reasonably sparse \textit{oracle graph}.

\item \textbf{Candidate Subgraphs.} With the oracle graph in hand, we then construct the set of all ball subgraphs, as defined in Section \ref{sec:loc} of our main paper. By limiting ourselves to only a few ($D|V|$) subgraphs, we can perform scan statistics more efficiently.

\item \textbf{Characterizing the Null Distribution.} In the case where we have few samples, we cannot directly apply the $\chi^2$ result. In these cases, the null distribution is then characterized using permutation testing over all candidate subgraphs. For each subgraph the input data is permuted a number of times to generate a good representation of the distribution at that subgraph. All normalized (but not size-corrected) scan statistics are then calculated for all permutations across all subsets and then combined in order to create the null distribution.

\item \textbf{Calculating the Test Statistic} For a specific subset of the data, the scan statistic is calculated and corrected as described in Section \ref{sec:loc} of the main paper, over the original grouping of the data. For each group, the logitudinal-covariance GLM \eqref{eq:lcglm} is computed using the procedures in \S\ref{sec:effest}.
%
\item \textbf{Region Identification.} We first identify all subsets whose statistic falls above the $\alpha$-level threshold specified. Then the subset-collection procedure outlined in the main paper, developed by \cite{jeng2010optimal}, is applied, and the non-overlapping critical regions are output.
\end{enumerate}

\subsection*{Numerical Considerations} 
In practice, our empirical covariance matrices calculated on the sample data may not be positive definite. The matrix can be rank deficient when we do not have enough linearly independent samples. 
In addition, we may use a rank correlation matrix in its place, which also may not be PD.
To resolve this issue, we \textit{project} the empirical covariance matrix onto the symmetric-positive definite $\SPD(n)$ manifold. We first apply a standard procedure for transforming a 
symmetric matrix into a symmetric positive semidefinite (SPSD) one. As described in \cite{wu2005analysis}, the standard eigenvalue thresholding, or clipping, $\lambda_{SPSD} = \max(0,\lambda)$ is sensible 
since it provides the optimal projection of any matrix onto the SPSD manifold. 
Let $\Sigma = U\Lambda U^\top$ be the eigenvalue decomposition of the matrix $\Sigma$. The SPSD projection of $\Sigma$ is 
then $\text{proj}_{SPSD}(\Sigma) = U\text{diag}(\max(\lambda_1,0),\ldots,\max(\lambda_n,0))U^\top$. And so to project to the $\SPD(n)$ manifold we can simply add some epsilon to each element of the diagonal: 
\begin{align}
\text{proj}_{SPD}(\Sigma) = U\text{diag}(\max(\lambda_1,0),\ldots,\max(\lambda_n,0))U^\top + \epsilon I
\end{align}
A remark on the term $\epsilon I$ will be useful here. We find that in experiments, numerical problems can arise if the smallest eigenvalue of the projected matrix is too small. 
By iteratively adding a 
small $\epsilon$ until the smallest eigenvalue is above our threshold, we ensure that the matrix is positive definite for the exponential and logarithmic maps. They are necessary for moving back and forth between the manifold and the tangent space.
\subsection*{A note on localization accuracy}
In addition to simply checking whether or not we were able to correctly answer the hypothesis test group difference, it is important that if a significance is found, that it is found in the features that were originally used to generate the data. Using the same simulation setup as previous, we take the union of all subsets returned to be significant and check if each of the truly changing features $p_t$ are contained within the superset.

In this particular case we find that our localization is only dependent on the graphical lasso procedure we use to generate the oracle graph. As long as the sparsity specified is large enough to include at least $p_t$ edges, we find that in \textit{every} simulation where we find a significant difference, the features that express the difference are a superset of the true features.
\section{Preclinical AD Extended Details and Results.}
\subsection*{Data and Variable Descriptions}
In our neuroimaging experiments, a large number of our features describe specific and localized regions of the brain across multiple imaging modalities. Below we list and describe each of regions for each modality, and give a brief background on each of methods used to acquire the data. We also include the list of cognitive scores used in our analysis.

\subsubsection*{PET Imaging}
Positron emission tomography has become an increasingly popular method of imaging the brain, specifically in the areas where cognitive decline can be strongly correlated with the specific matter being imaged. Pittsburgh compound B (PiB) was used as the tracer for these images, and the 16 mirrored (Left and Right) regions labeled below were selected as strongly correlated with the development and progression of Alzheimer's Disease.
{\small
\begin{enumerate}
\item PiB Angular L/R
\item PiB Cingulum Ant L/R
\item PiB Cingulum Post L/R
\item PiB Frontal Med Orb L/R
\item PiB Precuneus L/R
\item PiB SupraMarginal L/R
\item PiB Temporal Mid L/R
\item PiB Temporal Sup L/R
\end{enumerate}
}

The average of the voxel values in each ROI (region of interest) of the brain are used for imaging features.
The 16 regions are highlighted in Figure \ref{fig:pibrois}.

\begin{figure*}[h]
\centering
  \includegraphics[width=0.8\textwidth]{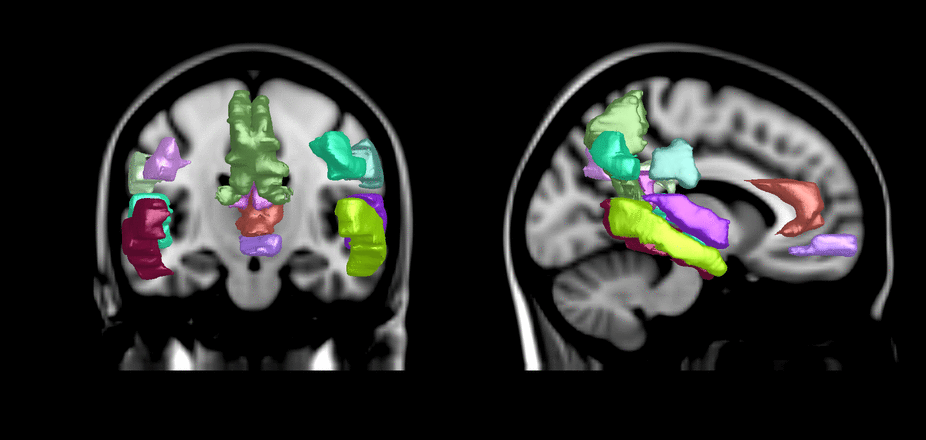}
  \caption{\label{fig:pibrois}16 Positron Emission Tomography (PET) regions.}
\end{figure*}

\newpage

\subsubsection*{DTI Imaging}
Diffusion tensor imaging is used to measure the restricted diffusion of water through and about regions of the brain. The 48 regions here are the aggregated measurements of total rates of diffusion for each voxel in that region. The two measurements, Fractional Anisotropy (FA) and Mean Diffusivity (MD) collectively well describe the diffusion in a specific region. The following is the full list of regions used in our analysis. Regions that spanned across both the left and right sides of the brain are indicated as such, and were treated as separate and independent in our analyses. 
{\small
\begin{multicols}{2}
\begin{enumerate}
\item Middle cerebellar peduncle
\item Pontine crossing tract (a part of MCP)
\item Genu of corpus callosum
\item Body of corpus callosum
\item Splenium of corpus callosum
\item Fornix (column and body of fornix)
\item Corticospinal tract R/L
\item Medial lemniscus R/L
\item Inferior cerebellar peduncle R/L
\item Superior cerebellar peduncle R/L
\item Cerebral peduncle R/L
\item Anterior limb of internal capsule R/L
\item Posterior limb of internal capsule R/L
\item Retrolenticular part of internal capsule R/L
\item Anterior corona radiata R/L
\item Superior corona radiata R/L
\item Posterior corona radiata R/L
\item Posterior thalamic radiation (include optic radiation) R/L
\item Sagittal stratum (include inferior longitidinal fasciculus and inferior fronto-occipital fasciculus) R/L
\item External capsule R/L
\item Cingulum (cingulate gyrus) R/L
\item Cingulum (hippocampus) R/L
\item Fornix (cres) / Stria terminalis (can not be resolved with current resolution) R/L
\item Superior longitudinal fasciculus R/L
\item Superior fronto-occipital fasciculus (could be a part of anterior internal capsule) R/L
\item Uncinate fasciculus R/L
\item Tapetum R/L
\end{enumerate}
\end{multicols}
}
\begin{figure*}
\centering
  \includegraphics[width=0.32\textwidth]{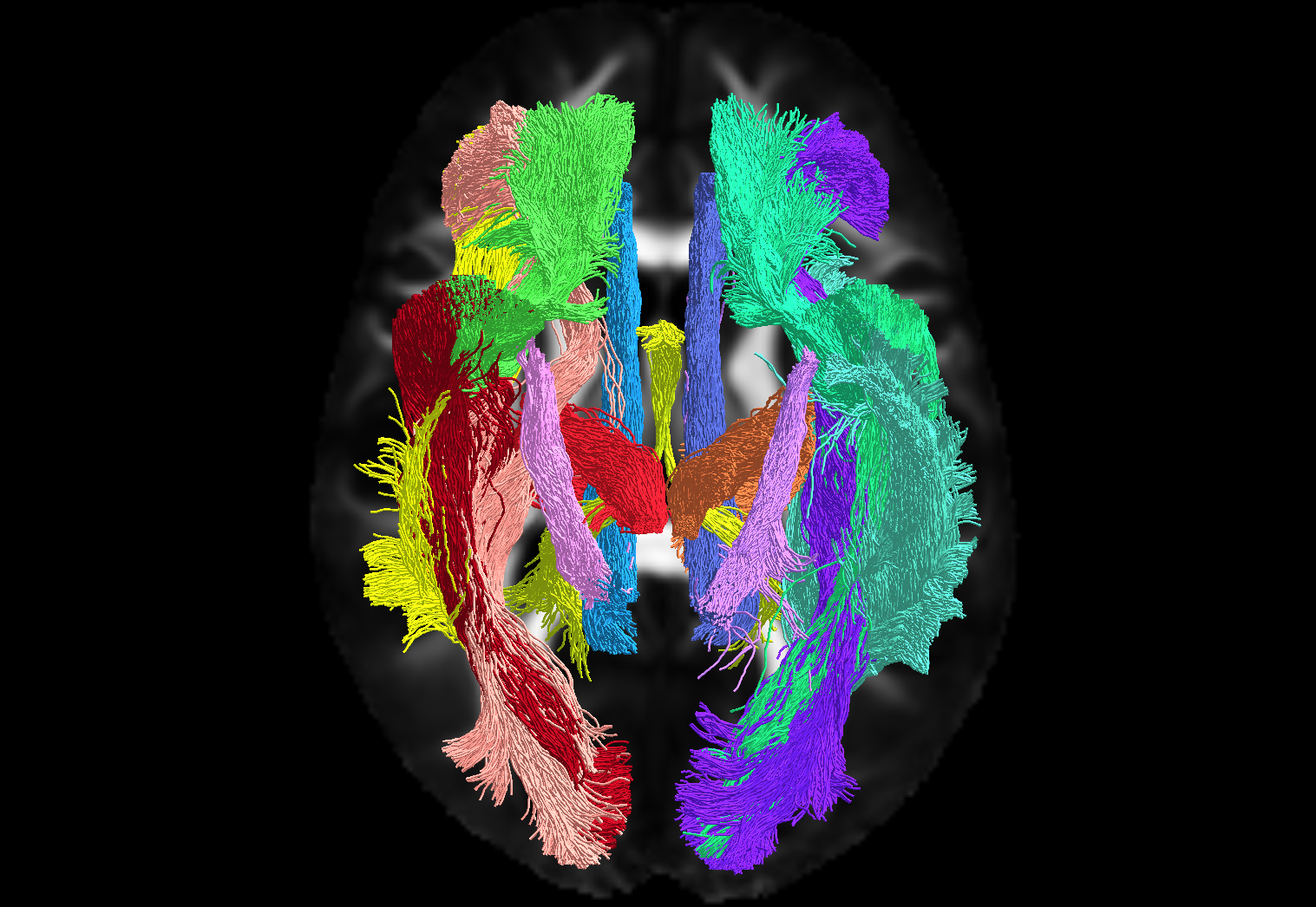}
  \includegraphics[width=0.32\textwidth]{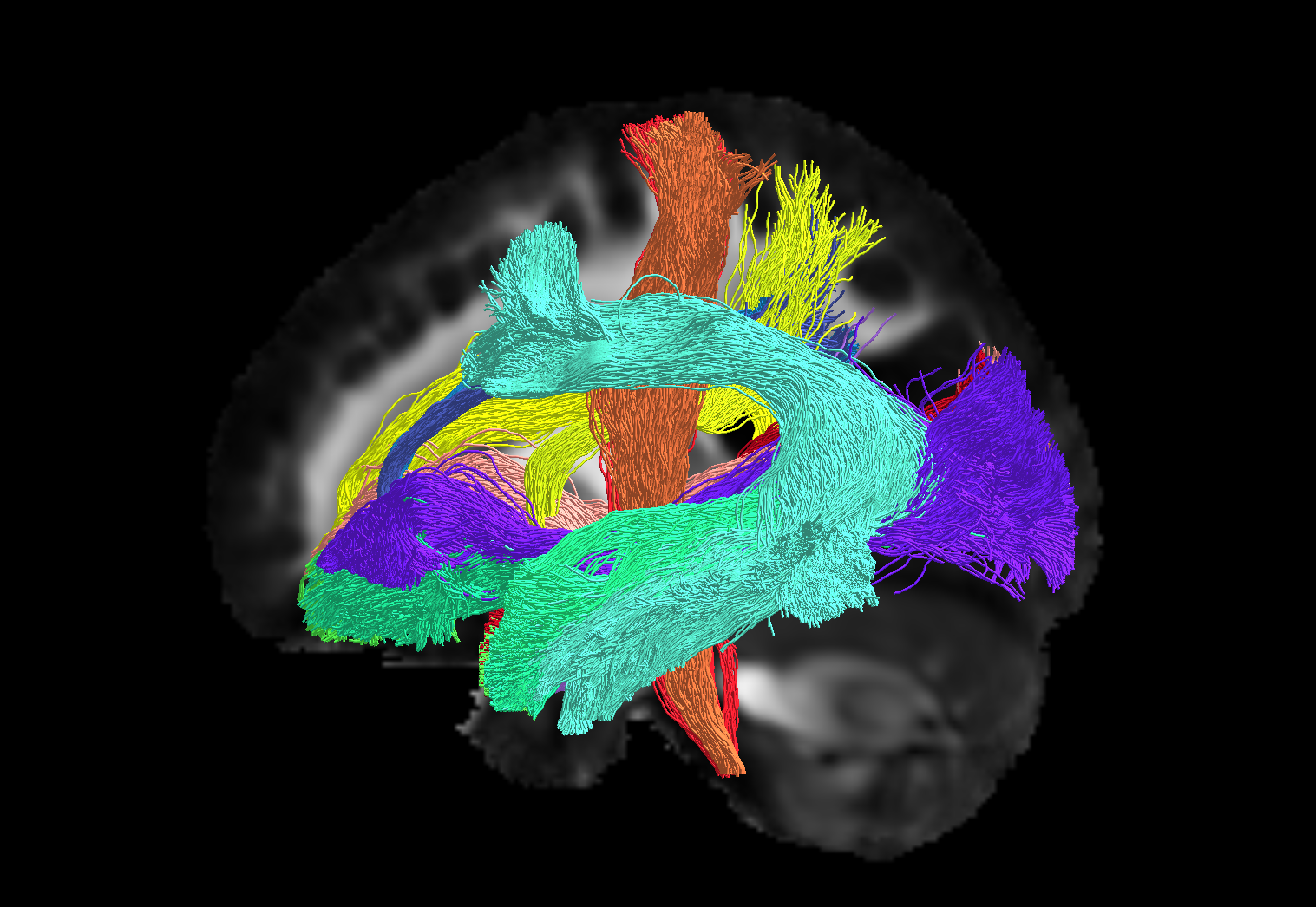}
  \includegraphics[width=0.32\textwidth]{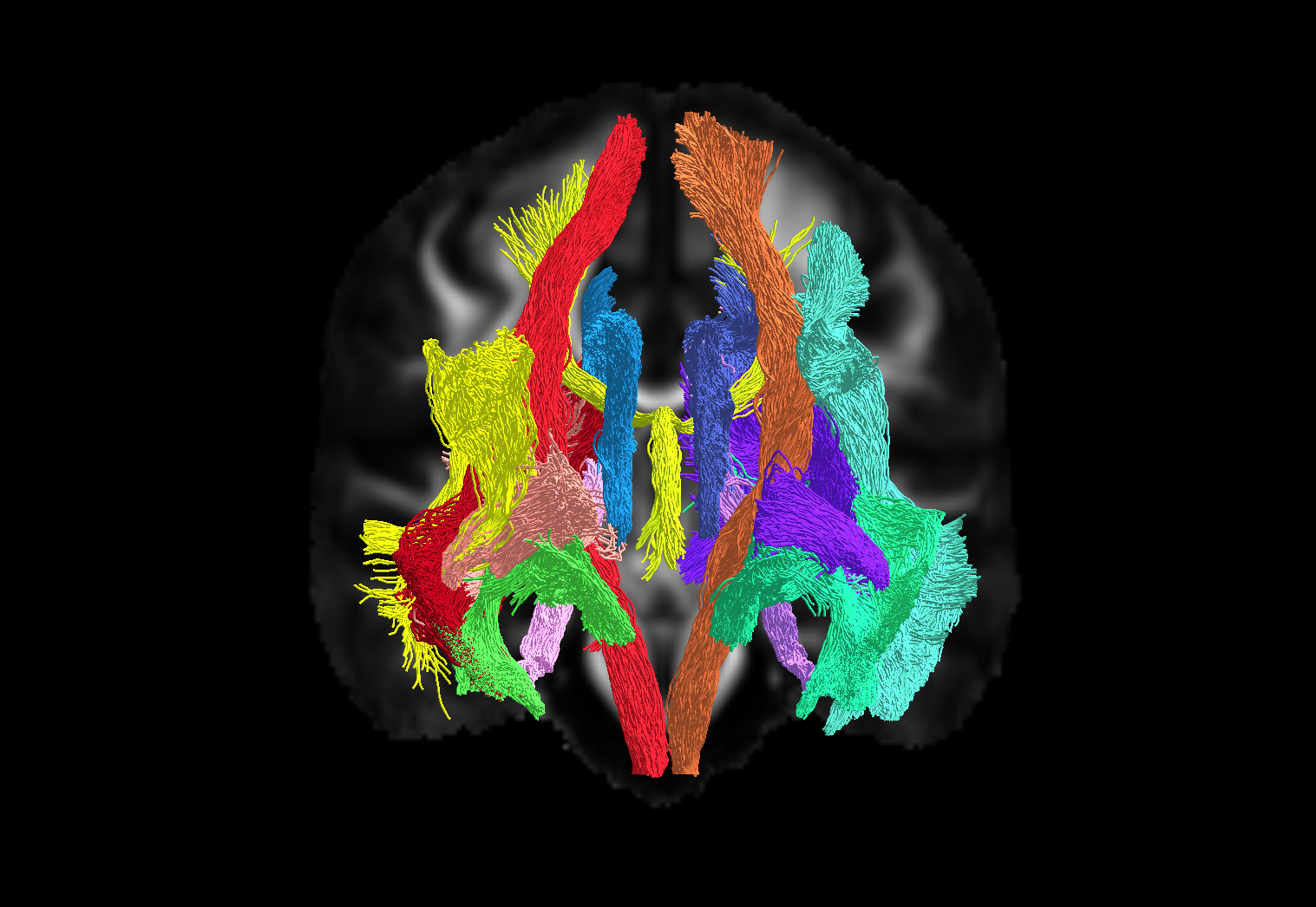}
  \caption{\label{fig:majorDTI}17 major DTI fiber bundles measured using Fractional Anisotropy (FA). The 48 selected for our analysis include a subset of these, which have been identified as critical regions that signal the beginnings of cognitive impairment.}
\end{figure*}

\newpage

\subsubsection*{Cognitive Evaluations}
The battery of cognitive test scores in our analysis included a breadth of evaluations chosen specifically for their coverage of various measures of cognition. Among all tests given to the cohort, the following 17 were selected by expert clinicians and researchers in the field for their coverage and their potential value in understanding trends across groups.
{\small
\begin{multicols}{2}
\begin{enumerate}
\item WAIS-III Digit Span Forward Raw Score
\item WAIS-III Digit Span Backward Raw Score
\item WAIS-III Letter-Number Sequencing Raw Score
\item COWAT CFL Score
\item Boston Naming Test Total Score
\item RAVLT Learning Trial A1 Raw Score
\item RAVLT Learning Trial A2 Raw Score
\item RAVLT Learning Trial A3 Raw Score
\item RAVLT Learning Trial A4 Raw Score
\item RAVLT Learning Trial A5 Raw Score
\item RAVLT Learning Trial A6 Raw Score
\item RAVLT Delayed Recall Raw Score
\item Stroop Word/Color-Word Scaled Score
\item Trail-Making Test Part A
\item Trail-Making Test Part B
\item Clock Drawing Test Score
\item Center for Epidemiologic Studies Depression Scale Score
\end{enumerate} 
\end{multicols}
}

{\bf WAIS-III.} This is the most widely used IQ test. The Digit Span examination is specifically meant to evaluate the working memory of an individual. Participants are required to attempt to recall a series of numbers in order, both forwards and backwards. Letter-Number sequencing reflects a similar idea, but with a mix of both numbers and letters in increasing and alphabetical order, and is meant to be an indicator of more complex mental control \cite{wechsler2014wechsler}.

{\bf Rey Auditory Visual Learning Test.} This test is specifically meant to evaluate all aspects of memory. Each trial evaluates a different type of memory, ranging from short-term and working memory to procedural and episodic memory. \cite{schmidt1996rey}.

{\bf Trail-Making Test.} This is a very popular test in providing information about executive function in the brain. The test consists of drawing lines among a randomly generated set of points in a square, where each point is labeled with a number. In Part A, participants must `connect the dots' in increasing numerical order, and in Part B in increasing numerical and alphabetical order. The score on the test is primarily dictated by the time in seconds it takes to complete the task for 25 of these `dots.' More background information and normative analyses can be found in \cite{tombaugh2004trail}.

Other tests similarly measure various cognitive function. While the Depression Scale Score did not crop up in any of our analyses here, it has been shown that depression is strongly associated with AD-related decline \cite{wragg1989overview}.
\subsection*{Detailed Imaging with Cognitive Tests Results}
In the following tables we provide additional details of the statistical test we performed on the preclinical AD cohort. Each set contains a set of features found to display significant group difference (at the $p \leq 0.05$ level) along the covariance trajectory divided by the group variable indicated.

While some of these associations are well-known, few have been indicated as novel by AD researchers and clinicians, and to be of interesting value for further analysis. 

\begin{table}[!]
	\centering
	\begin{tabular}{p{0.8cm}p{5.5cm}p{6cm}}
		\toprule
		\multicolumn{3}{c}{\textbf{Amyloid Load (PiB Positivity)}}\\ \midrule \midrule
		Set 1 & PiB Angular L/R & PiB Cingulum Ant L/R \\
		& PiB Cingulum Post L/R & PiB Frontal Med Orb L/R \\
		& PiB Precuneus L/R & PiB Temporal Sup L/R \\
		& PiB Temporal Mid L/R & \textbf{PiB SupraMarginal L} \\
		\midrule
		Set 2     & FA Cerebral peduncle R   & FA Cerebral peduncle L	\\
		& MD Corticospinal tract R	& MD Corticospinal tract L		\\
			     & Trail-Making Test Part A Score  & MD Cerebral peduncle R \\ 
			    &PET Cingulum Post R  &  \\ \midrule\bottomrule
	\end{tabular}
	\caption{Group difference across Amyloid Load (PiB Positivity)}
\end{table}

\begin{table}[!]
	\centering
	\begin{tabular}{p{0.8cm}p{5.5cm}p{6cm}}
		\toprule
		\multicolumn{3}{c}{\textbf{Gender}}\\ \midrule \midrule
        
		Set 1     & Rey Audio and Verbal Learning Test   & FA Cingulum  L	\\
		\midrule
		   & FA Medial lemniscus L	& FA Cingulum (hippocampus) L		\\
		Set 2    & FA Posterior thalamic radiation \newline (include optic radiation) L& \\ \midrule
    	Set 3    & FA Corticospinal tract R& FA Superior fronto-occipital fasciculus  R \\	\midrule\bottomrule
	\end{tabular}
    \caption{Group difference in gender}
\end{table}

\begin{table}[!]
	\centering
	\begin{tabular}{p{0.8cm}p{5.5cm}p{6cm}}
		\toprule
		\multicolumn{3}{c}{\textbf{Genotype: APOE4}}\\ \midrule \midrule
		Set 1    &Digit Span Backward Raw Score & Stroop Color-word \\ 
		 &PiB Cingulum Post L &    PiB Cingulum Post R         \\
		&PiB Frontal Med Orb L & 		PiB Frontal Med Orb R \\ 
		&PiB Precuneus L & PiB Precuneus R \\ 
		& PiB SupraMarginal &  PiB Temporal Mid R \\ \midrule		 \bottomrule
	\end{tabular}
	\caption{Group difference across Genotype APOE4 expression}
\end{table}

\begin{table}[!]
	\centering
	\begin{tabular}{p{0.8cm}p{5.5cm}p{6cm}}
		\toprule
		\multicolumn{3}{c}{\textbf{Consensus Conference}}\\ \midrule \midrule
		Set 2    &Digit Span Backward Raw Score & Stroop Color-word \\ 
		 &PiB Cingulum Post L &    PiB Cingulum Post R         \\
		&PiB Frontal Med Orb L & 		PiB Frontal Med Orb R \\ 
		&PiB Precuneus L & PiB Precuneus R \\ 
		& PiB SupraMarginal &  PiB Temporal Mid R \\ \midrule
		\bottomrule
	\end{tabular}
	\caption{Group difference across Expert MCI Diagnosis}
\end{table}

%

\subsection*{Detailed results on larger cohort with only cognitive scores}
We also applied our method to a larger cohort consisting of approximately 1500 subjects with varying temporal measurements on the battery of cognitive tests. Each individual had approximately 3 visits worth of data, and so our total number of measurements was approximately $n = 4000$. In addition to the groupings used above, we were able to use an algorithmic cognitive impairment (ACI) measure to further evaluate the model against a factor which is known to be group-separating. Below are the tabulated feature sets identified by our model for each of the group separations described in the main paper. In this case to increase interpretability of the results we limited our search to groups of 3-6 features.

When grouped by genotype, the most indicative subset as shown in Table \ref{fig:coggenotype}. These tests are most closely associated with memory, and we see that no tests of executive function or spatial ability (Trail-Making or Clock Drawing) were included.

In addition to an algorithmic measure of impairment, a conference of expert clinicians and researchers have given each individual a clinical impairment diagnosis for each time they underwent the cognitive battery. Using this as a group separator, we found a large number of overlapping subsets that displayed significant group difference at the $p = 0.05$ level. These are shown in Table \ref{fig:cogcc}. Trail-Making Test Parts A and B appeared in all identified subsets.

\begin{table}[h]
	\centering
	\begin{tabular}{p{0.8cm}p{5.5cm}p{6cm}}
		\toprule
		\multicolumn{3}{c}{\textbf{Algorithmic Cognitive Impairment}}\\ \midrule \midrule
		Set 1 & Boston Naming Test Total Score & RAVLT Learning Trial A1 Raw Score \\ 
		 & RAVLT Learning Trial A6 Raw Score &  \\ \bottomrule
		\bottomrule
	\end{tabular}
	\caption{Group Difference Localization Across Algorithmic Impairment}
\end{table}

\begin{table}[h]
	\centering
	\begin{tabular}{p{0.8cm}p{5.5cm}p{6cm}}
		\toprule
		\multicolumn{3}{c}{\textbf{Genotype: ApoE4}}\\ \midrule \midrule
		Set 1 & WAIS-III Digit Span Backward Raw Score & RAVLT Learning Trial A3 Raw Score \\
		& RAVLT Learning Trial A4 Raw Score & RAVLT Learning Trial A5 Raw Score \\
		\bottomrule
		\bottomrule
	\end{tabular}
	\caption{Group Difference Localization Across ApoE4 Genotype}
	\label{fig:coggenotype}
\end{table}

\begin{table}[h]
	\centering
	\begin{tabular}{p{5.5cm}p{6cm}}
		\toprule
		\multicolumn{2}{c}{\textbf{Expert Consensus Measure}}\\ \midrule \midrule
		WAIS-3 Letter-Number Sequencing Raw Score &
		Boston Naming Test Total Score \\
		RAVLT Learning Trial A2 Raw Score &
		RAVLT Learning Trial A3 Raw Score \\
		RAVLT Learning Trial A4 Raw Score &
		RAVLT Learning Trial A5 Raw Score \\
		RAVLT Learning Trial A6 Raw Score &
		RAVLT Delayed Recall Raw Score \\
		Trail-Making Test Part A &
		Trail-Making Test Part B \\
		Clock Drawing Test Score &
		Center for Epidemiologic Studies Depression Scale Score \\
		\bottomrule
		\bottomrule
	\end{tabular}
	\caption{Group Difference Localization Across Expert Clinical Diagnosis}
	\label{fig:cogcc}
\end{table}

\begin{figure}[hb!]
\centering
\includegraphics*[width=0.8\textwidth]{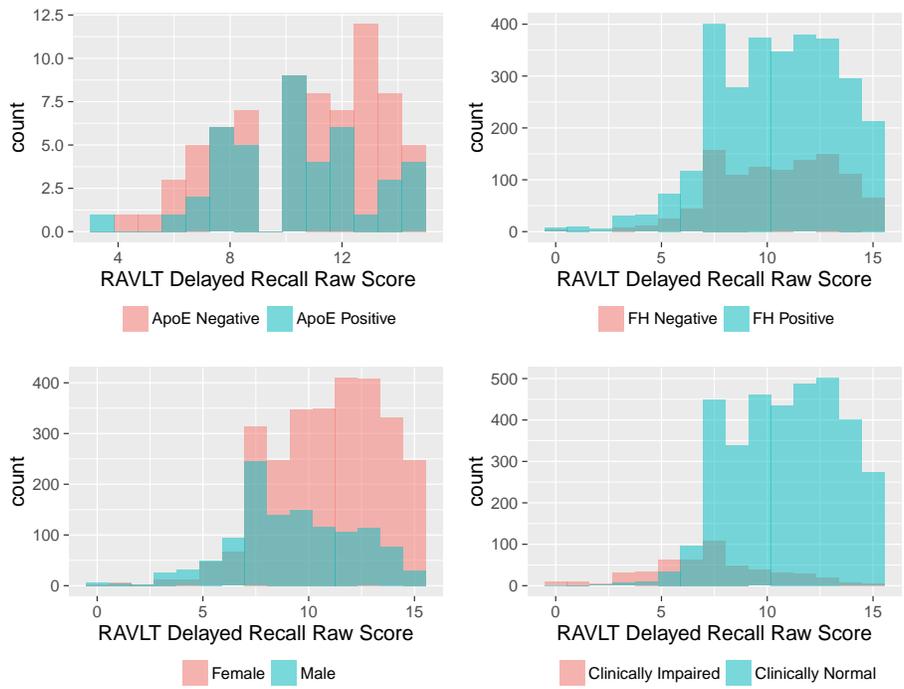}
\caption{Histograms of the Delayed Recall Scores for all time points for the $\sim 4000$ individual measurements across different group separations. We note in particular that the results found from the genotype separation above would have been hard to identify since given the distributions are extremely overlapping (top left) for this particular separation.}
\end{figure}
\clearpage
\section{Differential Geometry Basics and Notes.}
\label{sec:diffgeom}
We briefly introduce notions that we used in the main paper. For more details, we refer the reader to \cite{do1992riemannian,lee2003smooth,spivak1981comprehensive}.

\textbf{Differentiable manifold.}
A \textit{differentiable (smooth) manifold} of dimension $n$ is a set $\Mc$ and a maximal family of \textit{injective} mappings $\varphi_{i}:U_{i}
\subset \textbf{R}^{n} \rightarrow \Mc$ of open sets $U_{i}$ of
$\textbf{R}^{n}$ into $\Mc$ such that:
\begin{enumerate}
\item $\cup_{i}\varphi_{i}(U_{i}) =\Mc$
\item for any pair $i,j$ with $\varphi_{i} (U_{i}) \cap
\varphi_{j} (U_{j}) = W \neq \phi$, the sets $\varphi_{i}^{-1}(W)$
and $\varphi_{j}^{-1}(W)$ are open sets in $\textbf{R}^{n}$ and the
mappings $\varphi_{j}^{-1} \circ \varphi_{i}$ are
differentiable, where $\circ$ denotes function composition. 
\item The family $\{(U_{i},\varphi_{i})\}$ is maximal relative to
the conditions (1) and (2). 
\end{enumerate}

Roughly speaking, a differentiable (smooth) manifold $\Mc$ is a topological
space that is locally similar to Euclidean space and has a globally
defined differential structure. 

\textbf{Tangent space ($T_{p}\Mc$).} The \textit{tangent space} at $p \in \Mc$ is the vector space, which consists of 
the tangent vectors of {\em all} possible curves passing through $p$. 

\noindent\textbf{Tangent bundle ($T\Mc$).} The \textit{tangent bundle} of $\Mc$ is the disjoint union of tangent spaces at all points of $\Mc$, 
$T\Mc = \coprod_{p \in \Mc}T_{p}\Mc$. 
The tangent bundle is equipped with a natural \textit{projection map} $\pi: T\Mc \rightarrow \Mc$. 

\textbf{Riemannian manifold.} A \textit{Riemannian manifold} is 
equipped with a
smoothly varying metric (inner product), which is called \textit{Riemannian metric}. 

Various geometric notions, e.g., the angle between two curves or the length of a curve, can be extended on the manifold. \newline 

\textbf{Geodesic curves.} A geodesic curve on a Riemannian manifold is the locally shortest (distance-minimizing) curve.
These are analogous to straight lines in Euclidean space and a main object to generalize linear models to Riemannian manifolds.

\textbf{Geodesic distance.} The \textit{geodesic distance}
between two points on $\Mc$ is the length of the shortest {\em geodesic} curve connecting the two points. More generally, distance between two points on Riemannian manifolds is defined by the infimum of the length of all differentiable curves connecting the two points. Let $\gamma$ be a continuously differentiable curve $\gamma:[a,b] \rightarrow \Mc$ between $p$ and $q$ in $\Mc$ and $g$ be a metric tensor in $Mc$.
Then, formally, the distance between $p$ and $q$ is defined as
\begin{equation}
\text{d}(p,q) := \inf_\gamma \int_a^b \sqrt{g_\gamma(t) (\dot{\gamma}(t), \dot{\gamma}(t))} dt
\end{equation}
where $\gamma(a)=p$ and $\gamma(b)=q$.

\textbf{Exponential map}. An exponential map is a map from a tangent space $T_p\Mc$  to $\Mc$, which is usually locally defined due to the existence and uniqueness of ordinary differential equation for the map. The geodesic curve from $y_i$ to $y_j$ can be parameterized by a tangent vector in the tangent space at $y_i$ with an exponential map $\EXP(y_i,\cdot ): T_{y_i}\Mc \rightarrow \Mc$.

\textbf{Logarithm map.}
The inverse of the exponential map is the \textit{logarithm map}, $\LOG(y_i,\cdot):\M \rightarrow T_{y_i}\M$. 
For completeness, Table \ref{tab:comp} shows corresponding operations in the Euclidean space and Riemannian manifolds.
In the main paper, for the readability when operations are multiply nested, exponential map and its inverse logarithm map are denoted by $\EXP(p, x)$ and $\LOG(p, v)$ respectively, where $p, x \in \Mc$ and $v\in T_p\Mc$. They are usually denoted $\exp_p(x)$ and $\log_p(v)$ in most of differential geometry books. 
 
Separate from the above notations, matrix exponential, i.e, $\exp(X):= \sum \frac{1}{k!} X^k$, where $0!=1$ and $X^0=I$  and matrix logarithm are denoted by as $\exp(\cdot)$ and $\log(\cdot)$.

\renewcommand{\arraystretch}{1.5}
\begin{table}[!b]
{\footnotesize
\begin{center}
    \begin{tabular}{| l | l | l | }
    \hline
    Operation & Euclidean & Riemannian  \\  \hline 
    \footnotesize Subtraction & $\overrightarrow{x_i x_j} = x_j - x_i$ & $\overrightarrow{x_i x_j} = \LOG(x_i,x_j)$ \\ 
    \footnotesize Addition & $x_i + \overrightarrow{x_j x_k}$ & $\EXP(x_i,\overrightarrow{x_j x_k})$ \\     
    \footnotesize Distance  & $\| \overrightarrow{x_i x_j} \|$ & $\|\LOG(x_i,x_j) \|_{x_i}$ \\ 
    Mean  & $\sum_{i=1}^{n} \overrightarrow{\bar{x}x_{i}}=0$ & \footnotesize $\sum_{i=1}^{n} \LOG(\bar{x}, x_i)=0$  \\ 
    Covariance & \footnotesize$\EE \left [ (x_i - \bar{x})(x_i - \bar{x})^{T} \right ]$&\footnotesize $\EE \left [ \LOG(\bar{x}, x)\LOG(\bar{x}, x)^{T} \right ]$\\ [1ex] \hline 
  \end{tabular}
\end{center}
}
\caption{\footnotesize Basic operations in Euclidean space and Riemannian manifolds.}
\label{tab:comp}
\end{table}

\noindent {\bf Intrinsic mean.} 
Let $d(\cdot,\cdot)$ define the distance between two points. The intrinsic (or Karcher) mean is the minimizer to
{\small \begin{equation}
\label{eq:karchermean}
\bar{y} = \arg \min_{y \in \Mc} \sum_{i=1}^{N} d(y,y_{i})^{2}, 
\end{equation}}
which may be an arithmetic, geometric or harmonic mean depending on $d(\cdot,\cdot)$. A Karcher mean is a local minimum to \eqref{eq:karchermean} and a global minimum is referred as a Fr\'{e}chet mean. On manifolds, the Karcher mean satisfies $\sum_{i=1}^{N} \LOG_{\bar{y}}y_i =0$.

 \begin{figure}[H]
 \begin{center}
 \begin{minipage}{.45\linewidth} 
 \begin{algorithmic}[plain]
 \STATE \textbf{Algorithm 1 : Karcher mean}
 \STATE Input: $y_{1}, \ldots, y_{N} \in \M$, $\alpha$
 \STATE Output: $\bar{y} \in \M$
 \STATE $\bar{y}_{0} = y_{1}$
 \WHILE {$ \| \sum_{i=1}^{N} \LOG(\bar{y}_{k},y_{i})\| > \epsilon$}
 \STATE $\Delta\bar{y} = \frac{\alpha}{N} \sum_{i=1}^{N}\LOG (\bar{y}_k,y_i)$
 \STATE $\bar{y}_{k+1} = \EXP(\bar{y}_k,\Delta \bar{y})$
 \ENDWHILE
  \end{algorithmic}
  \end{minipage}
  \end{center}
 \caption{Karcher mean on manifolds}
     \label{alg:karcher} 
 \end{figure}
 
This identity implies the first order necessary condition of \eqref{eq:karchermean}, i.e., $\bar{y}$ is a local minimum with a zero norm gradient \cite{karcher1977riemannian}. In general, on manifolds, the existence and uniqueness of th.e Karcher mean is not guaranteed unless we assume, for uniqueness, that the data is in a small neighborhood.\\

\noindent {\bf Parallel transport.} 
Let $\Mc$ be a differentiable manifold with an affine connection $\nabla$ and $I$ be an open interval. Let $c:I \rightarrow \Mc$ be a differentiable curve in $\Mc$ and let $V_0$ be a tangent vector in $T_{c(t_0)}\Mc$, where $t_{0} \in I$. 
Then, there exists a unique parallel vector field $V$ along $c$, such that $V(t_0)=V_0$. Here, $V(t)$ is called the \textit{parallel transport} of $V(t_0)$ along $c$. 

\subsection*{Geometry of SPD manifolds}
Covariance matrices are symmetric positive definite matrices. 
Let SPD($n$) be a manifold for symmetric positive definite matrices of size $n\times n$. This forms a quotient space $GL(n)/O(n)$, where
$GL(n)$ denotes the general linear group (the group of $(n \times n)$ nonsingular matrices) and $O(n)$ is the orthogonal group 
(the group of $(n \times n)$ orthogonal matrices). 
The inner product of two tangent vectors $u,v \in T_{p}\Mc$ is given by 
\begin{equation}
\begin{split}
  \langle u,v \rangle_{p} = \tr(p^{-1/2}up^{-1}vp^{-1/2})
\end{split}
\label{eq:metricSPD}
\end{equation}
This plays the role of the Fisher-Rao metric in the statistical model of multivariate distributions.
Here, $T_{p}\Mc$ is a tangent space at $p$ (which is a vector space) is the space of symmetric matrices of dimension $(n+1)n/2$.
The geodesic distance is $d(p,q)^{2} = \tr( \log^{2}(p^{-1/2}qp^{-1/2}))$.

The exponential map and logarithm map are  given as 
\begin{equation}
\begin{split}
  \EXP(p,v) = p^{1/2} \exp(p^{-1/2}vp^{-1/2})p^{1/2}, \;\;
  \LOG(p,q) = p^{1/2} \log(p^{-1/2}qp^{-1/2})p^{1/2}.
\end{split}
\end{equation}

Let $p, q$ be in SPD($n$) and a tangent vector $w \in T_{p}\Mc$, the
tangent vector in $T_{q}\Mc$ which is the parallel transport of $w$ along
the shortest geodesic from $p$ to $q$ is given by 
\begin{equation}
\begin{split}
\Gamma_{p \rightarrow q}(w) &= p^{1/2}rp^{-1/2}wp^{-1/2}rp^{1/2} \\
\text{where } r &= \exp \left (p^{-1/2}\frac{v}{2}p^{-1/2} \right ) \text{ and
}v = \LOG(p,q)
\end{split}
\end{equation}


\newpage
\bibliography{main.bib}

\end{document}